\pdfoutput=1

\documentclass{article} %
\usepackage{iclr2027_conference,times}
\iclrfinalcopy %

\usepackage[utf8]{inputenc}
\usepackage[T1]{fontenc}
\usepackage{hyperref}
\usepackage{xurl}
\usepackage{microtype}
\usepackage{xcolor}

\usepackage{amsmath}
\usepackage{amsthm}
\usepackage{mathtools}
\usepackage{enumitem}
\usepackage{forloop}
\usepackage{xspace}
\usepackage{graphicx}
\usepackage{wrapfig}
\usepackage{amssymb}
\usepackage{thmtools}
\usepackage{thm-restate}
\usepackage{cleveref}
\usepackage{tikz}
\usepackage[most]{tcolorbox}

\usepackage{booktabs}
\usetikzlibrary{positioning,fit,shapes,arrows.meta,calc,backgrounds}

\declaretheorem[name=Proposition]{proposition}

\declaretheorem[name=Definition]{definition}

\newcommand{\defvec}[1]{\expandafter\newcommand\csname v#1\endcsname{{\mathbf{#1}}}}
\newcounter{ct}
\forLoop{1}{26}{ct}{
	\edef\letter{\alph{ct}}
	\expandafter\defvec\letter
}
\forLoop{1}{26}{ct}{
	\edef\letter{\Alph{ct}}
	\expandafter\defvec\letter
}

\DeclareMathOperator{\diag}{diag}
\DeclareMathOperator{\vecOp}{vec}

\DeclareMathOperator*{\argmin}{\arg\!\min}

\newcommand{\field}[1]{\ensuremath{\mathbb{#1}}}
\newcommand{\reals}{\field{R}}

\newcommand{\vphi}{\boldsymbol{\phi}}

\usepackage{colortbl}

\definecolor{C1}{HTML}{8898aa}   %
\definecolor{C2}{HTML}{1a5fa0}   %
\definecolor{C3}{HTML}{2CA02C}
\definecolor{C4}{HTML}{D62728}
\definecolor{C5}{HTML}{9467BD}
\colorlet{C1light}{C1!70!white}
\colorlet{C2light}{C2!70!white}
\colorlet{C3light}{C3!70!white}
\colorlet{C4light}{C4!70!white}
\colorlet{C5light}{C5!70!white}
\colorlet{C1lighter}{C1!50!white}
\colorlet{C2lighter}{C2!50!white}
\colorlet{C3lighter}{C3!50!white}
\colorlet{C4lighter}{C4!50!white}
\colorlet{C5lighter}{C5!50!white}
\colorlet{C1vlight}{C1!20!white}
\colorlet{C2vlight}{C2!20!white}
\colorlet{C3vlight}{C3!20!white}
\colorlet{C4vlight}{C4!20!white}
\colorlet{C5vlight}{C5!20!white}

\newcommand{\E}{\mathbb{E}}

\definecolor{mdcolor}{rgb}{1, 0.1, 0.59}

\newtcolorbox{keyresult}{%
  enhanced,
  colback=C1!4!white,
  colframe=C1!45!black,
  boxrule=0.6pt,
  arc=2pt,
  left=8pt, right=8pt, top=4pt, bottom=4pt,
  before skip=8pt, after skip=8pt,
}

\newtcolorbox{takeawaybox}[1][]{%
  enhanced,
  breakable,
  colback=white,
  colframe=C1!50!black,
  colbacktitle=C1!10!white,
  coltitle=black,
  fonttitle=\bfseries,
  title={#1},
  boxrule=0.75pt,
  arc=2pt,
  left=8pt, right=8pt, top=5pt, bottom=5pt,
  before skip=10pt, after skip=10pt,
}

\title{Memory by Design:\\Bayesian Layers for Sequence Modeling}

\author{%
  Matthew Dowling$^{1,2,4}$ \quad
  Hyungju Jeon$^{1}$ \quad
  Cristina Savin$^{2,3}$ \quad
  Il Memming Park$^{1,4}$
  \\[0.5em]
  $^{1}$Champalimaud Research, Champalimaud Foundation, Portugal \\
  $^{2}$Center for Neural Science, New York University, USA \\
  $^{3}$Center for Data Science, New York University, USA \\
  $^{4}$RyvivyR Inc., NY, USA
}

\definecolor{hjcolor}{rgb}{0.25, 0.75, 0.5}         %

\definecolor{mpcolor}{rgb}{1, 0.1, 0.59}         %

\begin{document}
\maketitle
\lhead{Preprint. Under review.} %
\begin{abstract}
To battle softmax attention's linearly growing inference memory cost,
current recurrent architectures compress history into a fixed-dimensional state.
Their recurrence equations alone, however,
leave the underlying assumptions about memory implicit.
We introduce \emph{design models}: auxiliary probabilistic models specified by an initial distribution and input-dependent factors encoding memory dynamics and incoming evidence.
Once these assumptions are specified, Bayesian filtering determines the recurrent update under the model.
Specifically, a linear-Gaussian design model yields the \textbf{Bayesian Layer},
which maintains a memory estimate and its covariance via a Kalman update.
Resetting this covariance recovers Delta-rule updates, while a different design model yields the additive recurrences of linear attention, GLA, and Mamba-2.
These connections make explicit the assumptions underlying familiar memory updates.
We illustrate how the geometry of the Bayesian Layer's uncertainty shapes subsequent updates and helps preserve earlier associations.
In associative recall tasks, the Bayesian Layer improves recall on longer sequences and with greater overlap between keys than seen during training.
We further evaluate its performance on multi-query associative recall (MQAR) and language modeling.
Distilling Bayesian Layers into a pretrained Gated DeltaNet improves retrieval accuracy on RULER over a continued-training control.
\end{abstract}

\section{Introduction}
Efficient recurrent sequence layers compress a growing token sequence into a
finite-dimensional state that is updated as new inputs arrive. Deciding how to
update this state is a central problem in recurrent architecture design: each
update must incorporate new information while preserving information from earlier
inputs that remains relevant. Linear attention, GLA, Mamba, and Delta-rule layers
address this problem through different recurrence
equations~\citep{Katharopoulos2020-cz,yang2024gated,gu2024mamba,schlag2021linear,
sun2023retnet,qin2024hgrn2,liu2025longhorn,yang2025gated,
siems2025deltaproduct,peng2023rwkv,peng2024rwkv,behrouz2026atlas,
de2024griffin,beck2024xlstm}, while other work derives such rules from online
regression and Bayesian
inference~\citep{vonoswald2026mesa,bonnet2026learning,shaj2026kalman}. The recurrence
equations alone, however, do not make explicit the assumptions about what is
stored and how new evidence relates to it.

We make these assumptions explicit through a \emph{design model}: an auxiliary
probabilistic model of a latent memory state. The model specifies an initial
distribution and input-dependent factors that encode state dynamics and incoming
evidence. Consequently, Bayesian filtering determines the recurrent update from
these assumptions. The belief parameters needed for subsequent updates and
readouts form the implemented recurrent state, while a separate query-dependent
readout produces the layer output. This framework thus links a
recurrent update to explicit assumptions about memory.

The \emph{Bayesian Layer} (BL) instantiates this framework with linear-Gaussian
latent-memory dynamics and write likelihood. Its recurrent state contains the
posterior mean of a latent memory and a dense covariance that tracks uncertainty
over memory directions. Exact filtering uses this covariance to set the
direction and gain of each write. Repeated evidence therefore reduces the gain
of later corrections along directions that have become less uncertain.

The framework also provides a way to relate existing recurrent architectures
through the assumptions underlying their memory updates
(\Cref{tab:connections}). Resetting the Bayesian Layer's predicted covariance to
an isotropic matrix at each step recovers Delta-rule
updates~\citep{schlag2021linear,yang2024parallelizing,yang2025gated,
kimi2025linear}. Under a separate design model, exact filtering yields the
additive recurrences of linear attention~\citep{Katharopoulos2020-cz},
GLA~\citep{yang2024gated}, and Mamba-2/SSD~\citep{dao2024transformers}. These
connections distinguish approximations to the filtering recurrence from
differences in the underlying design model.

Our contributions are:
\begin{enumerate}
  \item We introduce the \textbf{design-model framework}, which derives memory
  updates for recurrent architectures through Bayesian filtering and makes
  explicit the assumptions behind Delta-rule and additive recurrences
  (\S\S\ref{sec:framework}--\ref{sec:specializations}).
  \item We derive the \textbf{Bayesian Layer} and analyze how tracking
  uncertainty shapes the geometry of memory updates, helping preserve
  associations under interference (\S\ref{sec:framework}).
  \item We evaluate the layer on \textbf{associative recall tasks, MQAR, and
  language modeling}, including long-context retrieval after distillation into
  a pretrained Gated DeltaNet (\S\ref{sec:experiments}).
\end{enumerate}

\section{The design-model framework}
\label{sec:framework}

We write a recurrent sequence layer in terms of its update function $f$.
At each step, the current input $\vx_t$ and previous state $\vh_{t-1}$
determine the updated state,
\begin{equation}
  \vh_t = f(\vh_{t-1},\vx_t).
  \label{eq:layer_update}
\end{equation}
Choosing $f$ determines how each new input modifies a finite-dimensional
representation of the past. Modern recurrent architectures typically specify the
memory-update rule \(f\) directly through computational structure that supports
parallel training and mechanisms such as gating and error-correcting updates
\citep{dao2024transformers,yang2025gated}. These choices shape how information
accumulates in memory, but their underlying assumptions can be difficult to
identify and compare from the recurrences alone.

We propose making these assumptions the starting point for architectural design,
through a probabilistic model of memory, which we call a design model.
The model specifies an initial distribution and input-dependent factors that
encode memory dynamics and incoming evidence. Bayesian filtering determines the
optimal recurrent update under the model~\citep{sarkka2013bayesian}, providing a systematic way to construct
updates and interpret their behavior.

Let $\vS_t$ denote the latent memory state. A design model specifies an initial
distribution $\mu_0(\vS_0)$ and nonnegative, input-dependent pairwise potentials
$\Psi_t(\vS_{t-1},\vS_t;\vx_t)$. For each input sequence $\vx_{1:t}$, these define
\begin{equation}
  \pi_t(\vS_{0:t};\vx_{1:t})
  =\frac{1}{Z_t(\vx_{1:t})}\,
    \mu_0(\vS_0)\prod_{s=1}^{t}\Psi_s(\vS_{s-1},\vS_s;\vx_s),
  \label{eq:design_prefix}
\end{equation}
where $Z_t(\vx_{1:t})$ is a strictly positive normalizing constant.
The inputs parameterize this distribution; inference is performed over the
latent states. The design model does not specify a probability distribution
for the raw input sequence $\vx_{1:t}$.

Marginalizing the preceding states gives the filtering distribution $\mu_t$
and its Bayesian filtering recursion~\citep{murphy2023probabilistic}:
\begin{equation}
  \mu_t(\vS_t)
  =\int \pi_t(\vS_{0:t};\vx_{1:t})\,d\vS_{0:t-1}
  \propto
  \int \Psi_t(\vS_{t-1},\vS_t;\vx_t)
  \mu_{t-1}(\vS_{t-1})\,d\vS_{t-1}.
  \label{eq:design_filter}
\end{equation}
Once the factors are specified, this recursion determines how each new input
updates the belief over memory.

To obtain a layer output from this posterior, we choose a readout function
$g(\vS_t,\vx_t)$ specifying the quantity to estimate from the state. Under
squared-error loss, its posterior mean gives the Bayes-optimal output:
\begin{equation}
  \vy_t
  = \argmin_{\va}\;\E_{\mu_t}
    \!\left[\|\va-g(\vS_t,\vx_t)\|^2\right]
  = \E_{\mu_t}\!\left[g(\vS_t,\vx_t)\right].
  \label{eq:design_readout}
\end{equation}
Thus the design model determines the filtering update, while the readout
specifies what the layer estimates.

With the model and readout specified, a recurrent implementation requires the
belief parameters needed for subsequent filtering updates and readouts to admit
a closed, fixed-size recurrence. These parameters form the recurrent state
$\vh_t$ in \eqref{eq:layer_update}, whose update is determined by filtering.

The architectural choices in this framework are therefore the initial
distribution $\mu_0$, the form of the input-dependent pairwise potentials
$\Psi_t$, and the readout $g$ (\Cref{fig:framework}). In the following sections,
we illustrate these choices with the Bayesian Layer design model, whose factors encode dynamics and
associative relationships, and a latent-input design model, whose factors encode
evidence about independent write coefficients.

\begin{figure*}[t]
  \centering
  \includegraphics[width=\linewidth]{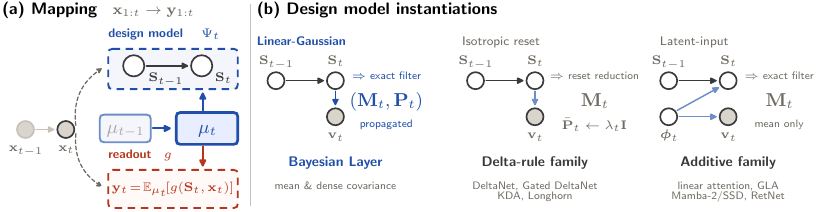}
  \caption{%
    \textbf{Design-model framework.}
    \textbf{(a)} Input parameterizes the design model, Bayesian filtering updates
    the belief over the latent memory state, and readout produces the layer output.
    \textbf{(b)} Different design models and filtering choices yield different
    recurrent updates. Exact filtering under the linear-Gaussian model gives the
    Bayesian Layer; isotropic covariance reset gives the \emph{Delta-rule family};
    exact filtering under the latent-input model gives the \emph{additive family}.
  }
  \label{fig:framework}
\end{figure*}

\subsection{The Bayesian Layer design model}
\label{sec:filter}

Matrix-memory architectures represent key--value associations through a
shared linear map, $\vk\mapsto\vS_t^\top\vk$
\citep{amari1972learning,yang2025gated,vonoswald2026mesa}. We use this representation to
construct the \emph{Bayesian Layer design model}, in which each value
$\vv_t$ is a noisy observation of $\vS_t^\top\vk_t$ and the memory
follows input-dependent dynamics between observations.

\begin{definition}[Bayesian Layer design model]
\label{def:design_model}
Let $\vS_t\in\reals^{D\times m}$. Learned functions of the input $\vx_t$
produce a key $\vk_t=\vk(\vx_t)\in\reals^D$, a value
$\vv_t=\vv(\vx_t)\in\reals^m$, a transition matrix
$\vA_t=\vA(\vx_t)\in\reals^{D\times D}$, and noise scales $r_t=r(\vx_t)>0$ and
$\ell_t=\ell(\vx_t)\geq0$.

The initial distribution $\mu_0$ has independent Gaussian columns with mean
zero and covariance $\sigma_0^2\vI_D$, where $\sigma_0^2>0$.
For $\ell_t>0$, we define
\begin{equation}
  \Psi_t(\vS_{t-1},\vS_t;\vx_t)
  =
  \exp\!\Biggl[
    -\underbrace{\frac{\|\vS_t-\vA_t\vS_{t-1}\|_F^2}{2\ell_t^2}}_{\text{dynamics}}
    -\underbrace{\frac{\|\vv_t-\vS_t^\top\vk_t\|_2^2}{2r_t^2}}_{\text{incoming evidence}}
  \Biggr].
  \label{eq:bl_factor}
\end{equation}
Equivalently, the design model has the joint representation
\begin{equation}
  p_t(\vS_t,\vv_t\mid\vS_{t-1})
  =p_t(\vS_t\mid\vS_{t-1})\,p_t(\vv_t\mid\vS_t)
  \propto\Psi_t(\vS_{t-1},\vS_t;\vx_t),
  \label{eq:bl_joint}
\end{equation}
where the transition has independent Gaussian columns with means
$\vA_t\vS_{t-1}^{(j)}$ and covariance $\ell_t^2\vI_D$, and
$p_t(\vv_t\mid\vS_t)=\mathcal N(\vv_t;\vS_t^\top\vk_t,r_t^2\vI_m)$.
Here the joint density is evaluated at the supplied key--value pair
$(\vk_t,\vv_t)$, with the input-dependent parameters held fixed.
\end{definition}

The two noise scales express different assumptions. The observation
variance $r_t^2$ describes discrepancies between an observed value
$\vv_t$ and the memory prediction $\vS_t^\top\vk_t$. The process
variance $\ell_t^2$ describes deviations of the memory from its expected
evolution $\vA_t\vS_{t-1}$. Thus an imperfect fit to an association and
a change in the underlying memory have distinct roles in the model.

For a given input, the dynamics and observations are linear-Gaussian in $\vS_t$
even when the learned functions of $\vx_t$ are nonlinear. Filtering therefore
preserves independent Gaussian columns with a shared covariance
(\Cref{app:bl_derivation}):
\begin{equation}
  \mu_t(\vS_t)
  =\prod_{j=1}^{m}
    \mathcal{N}\!\left(\vS_t^{(j)};\vM_t^{(j)},\vP_t\right).
  \label{eq:belief_form}
\end{equation}
The posterior mean $\vM_t\in\reals^{D\times m}$ and covariance
$\vP_t\in\reals^{D\times D}$, shared across columns, form the recurrent
state of the Bayesian Layer, $\vh_t=(\vM_t,\vP_t)$. Equivalently, $\vecOp(\vS_t)\sim
\mathcal{N}(\vecOp(\vM_t),\vI_m\otimes\vP_t)$.

As in self-attention \citep{vaswani2017attention}, we compute a query
$\vq_t=\vq(\vx_t)\in\reals^D$ from the current input. Here, the query specifies
the linear readout $g(\vS_t,\vx_t)=\vS_t^\top\vq_t$, so the quantity to estimate
is the memory projection along that query. Its posterior mean gives the
Bayes-optimal layer output under squared-error loss:
\begin{equation}
  \vy_t
  =\E_{\mu_t}\!\left[\vS_t^\top\vq_t\right]
  =\vM_t^\top\vq_t.
  \label{eq:readout}
\end{equation}
Each coordinate of $\vS_t^\top\vq_t$ has posterior variance
$\vq_t^\top\vP_t\vq_t$. Subsequent memory updates use $\vP_t$, while the
readout uses only the mean. We call the recurrent layer defined by this design model and
readout the \emph{Bayesian Layer}. 

\subsection{The Bayesian Layer recurrence}
\label{sec:update}

Under this design model, Bayesian filtering~\citep{sarkka2013bayesian} gives a closed-form recurrence
for the recurrent state $\vh_t=(\vM_t,\vP_t)$.
Starting from $\vM_0=\mathbf{0}$ and $\vP_0=\sigma_0^2\vI$ as specified in
\Cref{def:design_model}, the same update applies from $t=1$.
\begin{keyresult}
\begin{proposition}[Bayesian Layer recurrence]
\label{prop:update}
Under the Bayesian Layer design model (\Cref{def:design_model}) with belief structure
\eqref{eq:belief_form}, the uncertainty-warped key $\vu_t$ gives the update
direction, and the recurrent state $(\vM_t,\vP_t)$ satisfies
\begin{align}
  \bar{\vP}_t &= \vA_t\,\vP_{t-1}\,\vA_t^\top + \ell_t^2\vI,
  & &\textit{\small predicted covariance}
  \label{eq:predicted_cov} \\[3pt]
  \vu_t &= \bar{\vP}_t\,\vk_t,
  & &\textit{\small update direction} \\[3pt]
  \beta_t &= \bigl(r_t^2 + \vk_t^\top\vu_t\bigr)^{-1},
  & &\textit{\small innovation precision} \\[6pt]
  \vM_t &= \bigl(\vI - \beta_t\,\vu_t\,\vk_t^\top\bigr)\,
            \vA_t\,\vM_{t-1}
          + \beta_t\,\vu_t\,\vv_t^\top,
  & &\textit{\small mean update}
  \label{eq:mean_update} \\[6pt]
  \vP_t &= \bar{\vP}_t - \beta_t\,\vu_t\,\vu_t^\top.
  & &\textit{\small covariance update}
  \label{eq:cov_update}
\end{align}
\end{proposition}
\end{keyresult}

This recurrence is the exact Kalman filter for the design
model~\citep{kalman1960new,anderson1979optimal,kailath2000linear}.
It preserves the belief structure in \eqref{eq:belief_form}, with full
posterior covariance $\vI_m\otimes\vP_t$. The layer therefore propagates
one shared $D\times D$ covariance matrix alongside the mean.
\Cref{app:bl_derivation} derives the recurrence and establishes this closure.

The predicted covariance maps the key $\vk_t$ to the update direction $\vu_t$,
while $\beta_t$ scales the correction. These quantities follow from filtering
once the learned functions
$\vA,\ell,\vk,\vv,r$ are evaluated at $\vx_t$ and the previous belief is given.
Prediction adds process noise through $\ell_t^2\vI$ and conditioning then
reduces covariance by the rank-one term in \eqref{eq:cov_update}. Both steps
preserve positive semidefiniteness. 

For efficient computation, we restrict $\vA_t$ to be diagonal or blocks of
damped rotations; \Cref{app:chunk_parallel_filtering} gives the full
implementation details. Multiplication by
$\vA_t$ then uses row-wise scaling or transformations within each block.
Covariance prediction and conditioning cost $\mathcal{O}(D^2)$, and the mean
update costs $\mathcal{O}(Dm)$. The per-step recurrent cost is therefore
$\mathcal{O}(D^2+Dm)$.

Together with the readout $\vy_t=\vM_t^\top\vq_t$, this recurrence gives
the key--value--query form common to efficient recurrent
layers~\citep{schmidhuber1992learning,schlag2021linear}.
The key $\vk_t$ and value $\vv_t$ enter the filtering update,
while the query $\vq_t$ enters only the readout.
Here, the write rule is derived by exact filtering under the design model
rather than specified directly. The next section examines how covariance guides
memory updates.

\subsection{Geometry of Bayesian Layer memory updates}
\label{sec:anatomy}

The Bayesian Layer updates a memory shared by all keys, so incorporating
one association can change the readout at another key. This interference
is a known limitation of linear-attention memories
\citep{schlag2021linear,yang2025gated}. The two-key example in
\Cref{fig:hero} illustrates how propagating covariance can mitigate it.

Let $\bar{\vv}_t(\vk)=\bar{\vM}_t^\top\vk$ and $\hat{\vv}_t(\vk)=\vM_t^\top\vk$
denote estimated readouts before and after incorporating an observation, with
$\bar{\vM}_t=\vA_t\vM_{t-1}$. Before incorporating the observation, each
coordinate of the readout $\vS_t^\top\vk_A$ has variance $\bar\sigma_A^2 := \vk_A^\top\bar{\vP}_t\vk_A$.
For an observation $(\vk_A,\vv_A)$, define the \emph{cross-key interference}
as the change of the readout at $\vk_B$ induced by the mean update in \Cref{eq:mean_update}:
\begin{equation}
  \Delta_{B\mid A} := \hat{\vv}_t(\vk_B)-\bar{\vv}_t(\vk_B)
  =
  \frac{\vk_B^\top\bar{\vP}_t\vk_A}
       {r_t^2+\bar\sigma_A^2}
  \bigl(\vv_A-\bar{\vv}_t(\vk_A)\bigr).
  \label{eq:cross_key_readout}
\end{equation}
Define $\Delta_{A\mid A}$ analogously as the readout change at $\vk_A$ after the
same observation. Past observations constrain some memory directions more
strongly than others, shaping how predicted covariance $\bar{\vP}_t$ maps the
incoming key to an update direction (\Cref{fig:hero}b,c).

\begin{figure*}[t]
  \centering
  \includegraphics{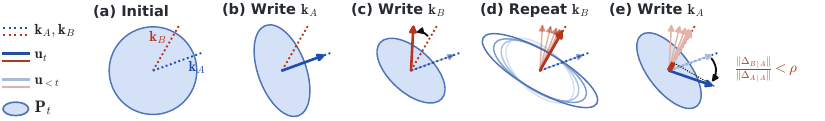}
  \caption{%
    \textbf{Covariance-guided memory updates.}
    Schematic with $D=2$, $\vA_t=\vI$, and unit keys with overlap
    $\rho\approx0.77$. Ellipses show posterior covariance $\vP_t$;
    dotted and solid lines show keys $\vk_t$
    and update directions $\vu_t$ respectively (faded: earlier).
    \textbf{(a)} Initial covariance.
    \textbf{(b--d)} Writes at $\textcolor[HTML]{1F4FA8}{\vk_A}$,
    then repeated writes at $\textcolor[HTML]{B5361C}{\vk_B}$
    (light to dark).
    \textbf{(e)} The next write at $\textcolor[HTML]{1F4FA8}{\vk_A}$;
    $\|\Delta_{B\mid A}\|/\|\Delta_{A\mid A}\|\approx0.246<\rho$}
  \label{fig:hero}
\end{figure*}

Consider unit keys $\vk_A$ and $\vk_B$ with overlap
$\rho=\vk_A^\top\vk_B\in[0,1)$. Decompose $\vk_A$ as
\[
  \vk_A=\rho\vk_B+\sqrt{1-\rho^2}\,\vk_\perp,
\]
where $\vk_\perp$ is a unit vector perpendicular to $\vk_B$.
The predicted variances along these directions are
$\bar\sigma_B^2=\vk_B^\top\bar{\vP}_t\vk_B$ and
$\bar\sigma_\perp^2=\vk_\perp^\top\bar{\vP}_t\vk_\perp$.
Under identity dynamics and fixed positive noise scales, repeated observations
at $\vk_B$ reduce $\bar\sigma_B^2$ relative to $\bar\sigma_\perp^2$
(\Cref{fig:hero}d).

When the next observation arrives at $\vk_A$, the covariance maps that key
to the write direction $\vu_t=\bar{\vP}_t\vk_A$.
When covariance is approximately diagonal in the basis $(\vk_B,\vk_\perp)$,
multiplying by $\bar{\vP}_t$ weights the parallel component $\rho\vk_B$ by
$\bar\sigma_B^2$ and the perpendicular component
$\sqrt{1-\rho^2}\,\vk_\perp$ by $\bar\sigma_\perp^2$.
We measure relative cross-key interference with \emph{confusion}, defined for
$\Delta_{A\mid A}\ne\mathbf{0}$ as
\begin{equation}
  \mathrm{confusion}
  :=\frac{\|\Delta_{B\mid A}\|}{\|\Delta_{A\mid A}\|}
  =\frac{|\vk_B^\top\bar{\vP}_t\vk_A|}{\vk_A^\top\bar{\vP}_t\vk_A}
  \approx\frac{\rho\,\bar\sigma_B^2}
              {\rho^2\bar\sigma_B^2+(1-\rho^2)\bar\sigma_\perp^2}.
  \label{eq:cross_key_change_ratio}
\end{equation}
Repeated observations at $\vk_B$ have reduced $\bar\sigma_B^2$ relative to
$\bar\sigma_\perp^2$, and this lowers confusion (\Cref{fig:hero}e).
As $\rho$ approaches one, the perpendicular contribution to the denominator
shrinks, so a larger variance contrast is needed to keep confusion low.
(See \Cref{app:key_confusion} for derivations and limits.)

Propagating dense covariance can therefore give the memory finer angular
resolution: stronger evidence, from repeated or lower-noise observations at a
stored key, helps keep the confusion defined in
\eqref{eq:cross_key_change_ratio} low between more closely aligned keys.
The Delta-rule and additive updates considered here instead write along the
incoming key, so for a nonzero write, confusion is fixed at the key overlap
$\rho$. Diagonal covariance provides an intermediate case: it can adapt the
write direction through coordinate-wise rescaling, but cannot represent
correlations between memory coordinates. The next section shows how resetting
covariance to an isotropic matrix recovers the Delta-rule family.

\section{Recurrent architecture families through design models}
\label{sec:specializations}

The construction in \S\ref{sec:framework} yields the Bayesian Layer by
exact filtering under the Bayesian Layer design model.
Resetting predicted covariance in the Bayesian Layer filter gives the
\emph{Delta-rule family}. A different model, the \emph{latent-input design
model}, treats observations as evidence about independent write coefficients
rather than stored associations. Exact filtering under this model gives
the \emph{additive family}.
\Cref{tab:connections} summarizes both families; derivations appear in
\Cref{app:architecture_connections}.

\begin{table}[!t]
\centering
\footnotesize
\caption{%
  \textbf{Recurrent architectures through design models.}
  Memory updates grouped by underlying design model.
  Top: the Bayesian Layer design model (\Cref{def:design_model}) yields
  the exact Bayesian Layer filter and, with covariance reset, Delta-rule
  updates (\Cref{prop:isotropic_reset}). Bottom: the latent-input model
  (\Cref{eq:latent_input_model}) yields exact additive updates.
  $\bar{\vM}_t := \vA_t \vM_{t-1}$,
  $\bar{\vv}_t := \bar{\vv}_t(\vk_t)=\bar{\vM}_t^\top \vk_t$,
  $\vu_t := \bar{\vP}_t \vk_t$,
  $\beta_t := (r_t^2 + \vk_t^\top \vu_t)^{-1}$,
  $\eta_t := \lambda_t / (r_t^2 + \lambda_t\|\vk_t\|^2)$, and
  $\omega_t := \lambda_t / (\lambda_t + r_t^2)$.
  $^\dagger$The Longhorn row gives its update before diagonal approximation.}
\label{tab:connections}
\renewcommand{\arraystretch}{1.1}
\setlength{\tabcolsep}{0.2pt}
\begin{tabular*}{\textwidth}{@{\extracolsep{\fill}} l l l @{}}
\toprule
\textbf{Model}
  & \textbf{Derivation}
  & \textbf{Recurrence} \\
\midrule
\multicolumn{3}{@{}l}{%
\textit{Bayesian Layer design model (\Cref{def:design_model}):}\;
$\vS_t \mid \vS_{t-1} \sim \mathcal{N}(\vA_t \vS_{t-1},\, \ell_t^2 \vI),\;
\vv_t \mid \vS_t \sim \mathcal{N}(\vS_t^\top \vk_t,\, r_t^2 \vI)$} \\[2pt]

\textbf{Bayesian Layer}
  & exact filter
  & $\vM_t=\bar{\vM}_t+\beta_t\vu_t(\vv_t-\bar{\vv}_t)^\top$
  \\

DeltaNet~\citep{schlag2021linear,yang2024parallelizing}
  & reset; $\vA_t=\vI$
  & $\vM_t=\bar{\vM}_t+\eta_t\vk_t(\vv_t-\bar{\vv}_t)^\top$
  \\

Gated DeltaNet (GDN)~\citep{yang2025gated}
  & reset; $\vA_t=\alpha_t \vI$
  & $\vM_t=\bar{\vM}_t+\eta_t\vk_t(\vv_t-\bar{\vv}_t)^\top$
  \\

KDA (Kimi Linear)~\citep{kimi2025linear}
  & reset; $\vA_t=\diag(\boldsymbol{\alpha}_t)$
  & $\vM_t=\bar{\vM}_t+\eta_t\vk_t(\vv_t-\bar{\vv}_t)^\top$
  \\

Longhorn$^\dagger$~\citep{liu2025longhorn}
  & per-column reset
  & $\vM_t^{(j)}=\vM_{t-1}^{(j)}+\eta_{t,j}\vk_t(v_{t,j}-\bar v_{t,j})$
  \\

\midrule
\multicolumn{3}{@{}l}{%
\textit{Latent-input design model (\Cref{def:latent_input_model}):}\;
$\vphi_t \sim \mathcal{N}(\mathbf{0},\, \lambda_t \vI),\;
\vS_t = \vA_t \vS_{t-1} + \vk_t \vphi_t^\top,\;
\vv_t \mid \vphi_t \sim \mathcal{N}(\vphi_t,\, r_t^2 \vI)$} \\[2pt]

Linear attention~\citep{Katharopoulos2020-cz}
  & exact; $\vA_t=\vI$
  & $\vM_t=\vM_{t-1}+\vk_t(\omega_t\vv_t)^\top$
  \\

RetNet~\citep{sun2023retnet}
  & exact; $\vA_t=\gamma \vI$
  & $\vM_t=\gamma\vM_{t-1}+\vk_t(\omega_t\vv_t)^\top$
  \\

GLA (diag.\ form)~\citep{yang2024gated}
  & exact; $\vA_t=\diag(\boldsymbol{\alpha}_t)$
  & $\vM_t=\diag(\boldsymbol{\alpha}_t)\vM_{t-1}+\vk_t(\omega_t\vv_t)^\top$
  \\

Mamba-2 / SSD~\citep{dao2024transformers}
  & equiv.\ form; $\vA_t=a_t\vI$
  & $\vM_t=a_t\vM_{t-1}+\vk_t(\omega_t\vv_t)^\top$
  \\

\bottomrule
\end{tabular*}
\end{table}

\subsection{\texorpdfstring{Bayesian Layer design model: covariance reset $\Rightarrow$ Delta-rule family}{Bayesian Layer design model: covariance reset => Delta-rule family}}
\label{sec:delta_family}

\begin{proposition}[Isotropic-reset reduction]
\label{prop:isotropic_reset}
Under the Bayesian Layer design model (\Cref{def:design_model}), fix a
scalar sequence $\lambda_t > 0$ and replace the predicted covariance
$\bar{\vP}_t$ in the filtering update by $\lambda_t\vI$ at every step.
Then $\vu_t = \lambda_t\vk_t$ and the Bayesian Layer update reduces to
\begin{equation}
  \vM_t
  =
  {(\vI - \eta_t\vk_t\vk_t^\top)\,\vA_t\,\vM_{t-1}}
  +
  {\eta_t\vk_t\vv_t^\top},
  \qquad
  \eta_t := \lambda_t / (r_t^2 + \lambda_t\|\vk_t\|^2).
  \label{eq:isotropic_reset}
\end{equation}
\end{proposition}
Equation~\eqref{eq:isotropic_reset} has the canonical Delta-rule form: the update
corrects the difference between the observed value and the predicted
memory readout at $\vk_t$. Covariance reset retains Delta-rule correction
but discards the covariance accumulated from past observations. Choices of
$\vA_t$ along the scalar--diagonal--dense hierarchy recover the variants in
\Cref{tab:connections}.

\subsection{\texorpdfstring{Latent-input design model: exact filter $\Rightarrow$ additive family}{Latent-input design model: exact filter => additive family}}
\label{sec:additive_family}
We consider a design model that treats each observation as
evidence for a new write, independent of stored associations.
\begin{definition}[Latent-input design model]
\label{def:latent_input_model}
Let the initial memory $\vS_0$ have mean $\vM_0$. Each observation
$\vv_t$ provides noisy evidence about a latent write coefficient
$\vphi_t\in\reals^m$. For input-dependent $\lambda_t,r_t>0$, define
\begin{equation}
  \vphi_t \sim \mathcal{N}(\mathbf{0}, \lambda_t\vI_m),
  \qquad
  \vv_t \mid \vphi_t \sim \mathcal{N}(\vphi_t, r_t^2\vI_m),
  \qquad
  \vS_t = \vA_t\vS_{t-1} + \vk_t\vphi_t^\top.
  \label{eq:latent_input_model}
\end{equation}
All write coefficients and observation noises are mutually independent
and independent of $\vS_0$.
\end{definition}

Given $\vS_{t-1}$ and $\vphi_t$, the memory $\vS_t$ is deterministic.
Marginalizing the write coefficient gives a design factor
$\Psi_t(\vS_{t-1},\vS_t;\vx_t)$ within the framework of \S\ref{sec:framework}.
Taking the posterior expectation of the state update in
\eqref{eq:latent_input_model} yields an additive recurrence for the latent
memory state's mean (\Cref{app:latent_input}).

\begin{proposition}[Exact additive recurrence]
\label{prop:additive_update}
Under the latent-input design model (\Cref{def:latent_input_model}),
the update for the posterior mean
$\vM_t := \E_{\mu_t}[\vS_t]$ is purely additive:
\begin{equation}
  \vM_t = \vA_t\vM_{t-1} + \omega_t\vk_t\vv_t^\top,
  \qquad
  \omega_t := \lambda_t / (\lambda_t + r_t^2).
  \label{eq:additive_update}
\end{equation}
\end{proposition}

Each observation contributes a new write without correcting the current memory
prediction. The exact mean update and readout require no covariance propagation,
although uncertainty in the write coefficients generally leaves nonzero
posterior uncertainty over memory (\Cref{app:latent_input}). Choices of $\vA_t$
recover the additive variants in \Cref{tab:connections}.

\section{Experiments}
\label{sec:experiments}

\begin{figure*}[t]
  \centering
  \includegraphics[width=\textwidth]{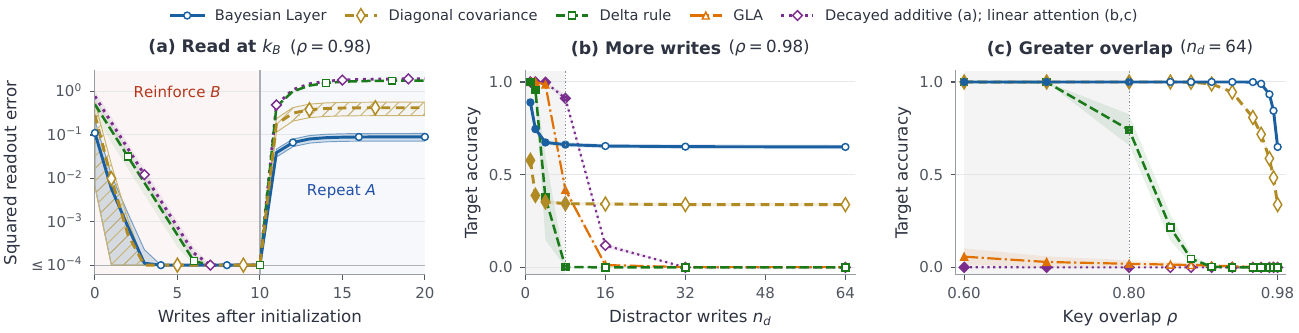}
  \caption[Associative recall with overlapping keys.]{%
  \textbf{Associative recall with overlapping keys.}
  \textbf{(a) Toy experiment without training.} Squared readout error at $\vk_B$
  during repeated writes at $\vk_B$ followed by $\vk_A$, with key overlap
  $\rho=0.98$. \textbf{(b,c) Learned associative recall.} Target recall accuracy
  versus distractor write count at $\rho=0.98$ (b) and versus key overlap after
  64 distractor writes (c). Bands show $\pm1$ standard deviation across initial
  write orderings in (a) and 3 seeds in (b,c). Grey shading and
  filled markers indicate the training range of the horizontal variable; the
  fixed variable is outside its training range in both panels. }
  \label{fig:controlled_collision}
\end{figure*}

We test whether propagating covariance helps preserve stored associations under
subsequent memory updates. We first use synthetic recall tasks with controlled
key overlap (\S\ref{sec:controlled_collision}), then evaluate recall on MQAR and
WikiText-103 (\S\ref{sec:zoology_mqar}) and retrieval after distilling Bayesian
Layers into a pretrained language model (\S\ref{sec:slimpajama}).
Appendix Table~\ref{tab:consolidated_settings} summarizes the experimental settings.

\subsection{Associative recall with overlapping keys}
\label{sec:controlled_collision}

The analysis in \S\ref{sec:anatomy} predicts that covariance geometry helps
preserve associations during writes at nearby keys. We first compare memory
recurrences with fixed keys and values, without learned projections. The toy
task stores six associations, labeled $A$--$F$, with overlap $\rho=0.98$ between
keys $A$ and $B$ while all other keys are orthogonal. After initialization, we
repeat $(\vk_B,\vv_B)$ ten times, followed by ten repetitions of
$(\vk_A,\vv_A)$. \Cref{fig:controlled_collision}a shows the squared readout
error $\|\hat{\vv}_t(\vk_B)-\vv_B\|^2$, averaged over initial write orderings
(\Cref{app:collision_details}).

Repeated writes at $\vk_B$ reduce all methods' mean error below $10^{-4}$, with
dense and diagonal covariance requiring fewer writes. However, during subsequent
writes at $\vk_A$, their performance diverges with dense covariance reaching the
lowest error plateau, followed by diagonal covariance
(\Cref{fig:controlled_collision}a). Asymptotic behavior under repeated writes is
analyzed in \Cref{app:cross_key_balance,app:panel_a_limits}.

We next test whether the covariance advantage persists in learned networks and
how it depends on key overlap. We separately train two-layer networks using the
Bayesian Layer, its diagonal-covariance variant,
Delta-rule~\citep{schlag2021linear},
linear-attention~\citep{Katharopoulos2020-cz}, and GLA~\citep{yang2024gated}
updates. Each training sequence contains target--distractor pairs $(T_i,D_i)$
with within-pair key overlap $\vk_{T_i}^\top\vk_{D_i}=\rho_i$. Keys from
different pairs remain orthogonal. Sequences initialize all associations, write
repeatedly at each target followed by $n_d$ writes at each distractor, and
finally query the target keys. We use the supplied key vectors directly as keys
and queries, so learned projections cannot change their overlap. Evaluation uses
more distractor writes and greater key overlap than training
(\Cref{app:synthetic_train_test_task,app:compared_models}).

At $\rho=0.98$, accuracy initially drops for both the Bayesian Layer and its
diagonal-covariance variant, then stays nearly constant, reaching 65.0\% and
33.9\%, respectively at $n_d=64$. In contrast, the other updates eventually fall
to near-zero accuracy. At $n_d=64$, dense covariance Bayesian Layer also
maintains high accuracy over a wider range of key overlap than any baseline
(\Cref{fig:controlled_collision}c). Both dense and diagonal variants benefit from
covariance propagation, but dense covariance better preserves stored
associations.

\subsection{Associative recall on MQAR and WikiText-103}
\label{sec:zoology_mqar}

The preceding experiment controls key overlap directly. We next ask how the
Bayesian Layer performs when the model must learn representations for storing
and retrieving associations from tokens. We begin with Multi-Query Associative
Recall (MQAR)~\citep{arora2024zoology}, where sequences present key--value
associations and later query their values. We compare two-layer models across
widths under a common data and training protocol (\Cref{app:zoology_details}).
The baselines are Gated DeltaNet~\citep{yang2025gated},
Gated DeltaNet-2 (GDN-2)~\citep{hatamizadeh2026gateddeltanet2},
MesaNet~\citep{vonoswald2026mesa}, RWKV-7~\citep{peng2025rwkv7},
Palimpsa-D~\citep{bonnet2026learning}, Mamba-2~\citep{dao2024transformers},
and causal attention~\citep{vaswani2017attention} with rotary position
embeddings~\citep{su2024roformer}.

\Cref{fig:mqar_wikitext}a reports recall at sequence length 1,024, beyond
the training maximum of 256. BL, MesaNet and
RWKV-7 achieve similarly high recall at width 256, but BL has higher recall than
those two baselines at width 128. BL therefore remains competitive with the strongest
recurrent baselines when evaluated on longer sequences containing more
associations.

\begin{wrapfigure}{r}{0.57\textwidth}
  \centering
  \includegraphics[width=\linewidth]{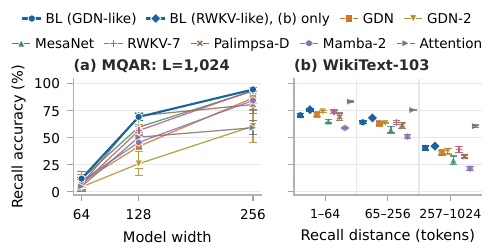}
  \caption{\textbf{Associative recall accuracy.}
  \textbf{(a) MQAR:} versus model width at length 1,024 (training maximum: 256),
  using GDN-like BL.
  \textbf{(b) WikiText-103:} for repeated rare token pairs, grouped by distance
  to their most recent occurrence.
  Means $\pm1$ standard deviation over 3 seeds.}
  \label{fig:mqar_wikitext}
\end{wrapfigure}

We then ask whether this recall performance extends to language modeling. We
train models with approximately 14--17M non-embedding parameters on
WikiText-103~\citep{merity2017pointer} for 520M tokens each, using 3 training
seeds. To isolate the contribution of the update
mechanism, we replace the recurrent updates in GDN and RWKV-7 with the Bayesian
recurrence while retaining their surrounding projections, normalization, and
feedforward layers. We call these the GDN-like and RWKV-like BL variants
(\Cref{app:from_scratch_wikitext}). Following the recall analysis of
\citet{arora2024zoology}, we measure associative recall on predictions
completing token pairs that are rare in the training corpus and have appeared
within the preceding 1,024 tokens of the same context. \Cref{fig:mqar_wikitext}b
groups these predictions by distance to the pair's most recent occurrence.

Both BL variants improve mean associative-recall accuracy over their
corresponding GDN and RWKV-7 baselines while maintaining competitive overall
perplexity. Their recall advantage is present in both longer-distance bins, and
the RWKV-like BL achieves the highest mean accuracy among recurrent models in
all three bins. Attention remains strongest on associative recall. The full
perplexity and recall results are reported in \Cref{tab:from_scratch_wikitext}.

\subsection{Distillation into a pretrained language model}
\label{sec:slimpajama}

\begin{figure}[h!]
	\centering
	\includegraphics[width=0.97\linewidth,trim=0 2bp 0 2bp,clip]{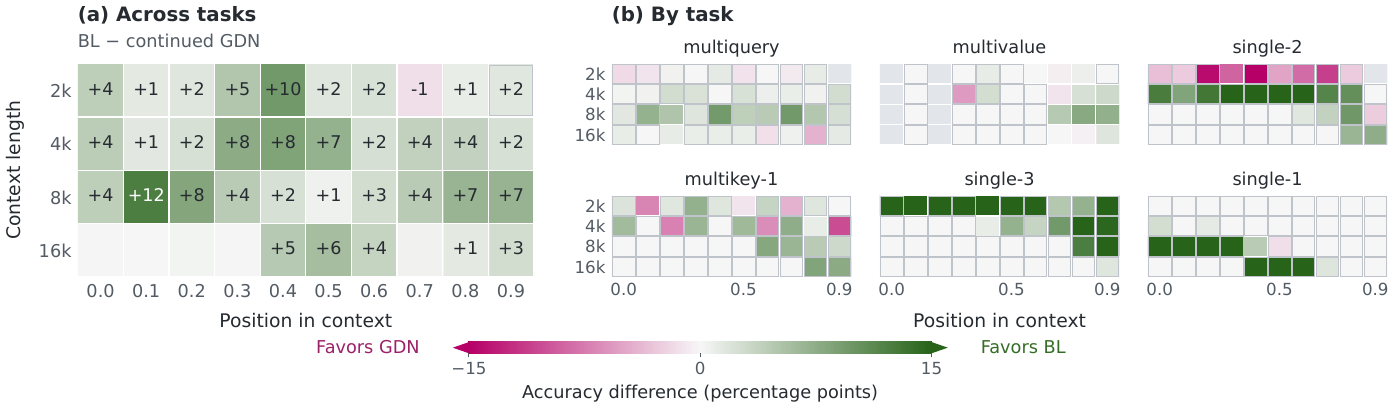}
	\caption{%
		\textbf{Long-context retrieval after layerwise distillation.}
		BL minus continued-GDN retrieval accuracy on RULER, \textbf{(a)} aggregated
	across six needle-in-a-haystack tasks and \textbf{(b)} separately for each
	task. Both methods are adapted on 2k-token sequences. Rows indicate evaluation
	context length; columns indicate the position of the information to be
	retrieved, from early to late. Results average 3 training seeds per method
	on the same evaluation examples. Gray outlines mark cells containing fewer
	than 30 examples; gray fill indicates empty cells. }
	\label{fig:slimpajama}
\end{figure}
Our last experiment examines whether Bayesian Layers can improve measured
retrieval accuracy in a pretrained language model. Building on work that adapts
pretrained models through recurrent layer
replacements~\citep{junxiongdaniele2024mambainllama,lan2025liger}, we replace
four of a 340M Gated DeltaNet's 24 sequence mixers with Bayesian Layers. Each
replacement is trained to match the corresponding frozen mixer's outputs, while
the remaining model parameters stay fixed.

To assess whether further training of the original architecture produces
comparable improvements, we also continue GDN with next-token cross-entropy,
updating the mixers and MLPs at the same four layers. Both procedures use
2k-token SlimPajama sequences and expose each selected mixer to 100M training
tokens, totaling 400M when summed across the four mixers. We evaluate three
training seeds per procedure on six needle-in-a-haystack tasks from
RULER~\citep{hsieh2024ruler}, varying context length from 2k to 16k and the
position of the information to be retrieved. Training details, complete results,
and the comparison with frozen GDN appear in \Cref{app:slimpajama_details}.

BL achieves higher mean retrieval accuracy than both frozen and continued GDN at
each evaluated context length. Across WikiText-103, held-out SlimPajama, and
PG-19, perplexity remains close to continued GDN, with increases of 0.3--0.7
points (approximately 3--4\%). Averaged across tasks, the retrieval gains over continued
GDN span every position bin at 4k and 8k (\Cref{fig:slimpajama}a). The task
breakdown shows improvements in both single-target and multi-target retrieval,
alongside a regression on single-2 at 2k (\Cref{fig:slimpajama}b). At
8k---four times the adaptation length---BL achieves 25.5\% mean accuracy
compared with 20.5\% for continued GDN.

\section{Discussion}
\label{sec:discussion}

\paragraph{Related probabilistic sequence layers.}
Like Palimpsa~\citep{bonnet2026learning} and Kalman Linear
Attention~\citep{shaj2026kalman}, the BL treats recurrent computation
as probabilistic inference. Its dense key-space covariance retains correlations
omitted by the diagonal representations used in their implemented layers. Dense
statistics also appear in MesaNet~\citep{vonoswald2026mesa}, where they support a
reweighted in-context regression solved by conjugate gradient at readout. The
BL instead propagates a posterior over an evolving state through
structured dynamics and process noise, with readout given by a linear projection
of the posterior mean of the latent memory state.

\paragraph{Limitations.}
Our geometric analysis focuses on Gaussian beliefs. Extending it to the
information geometry of other exponential families, drawing on natural-gradient
approaches to Bayesian learning, remains future work~\citep{khan2023bayesian}.
Dense covariance adds training computation quadratic in the key dimension and
additional activation storage, making low-rank, sparse, or other structured
approximations promising future directions. Finally, the readout uses only the
posterior mean, leaving the use of propagated uncertainty for downstream losses
or calibration unexplored.

\label{sec:conclusion}
\paragraph{Conclusion.}
Design models connect recurrent sequence architectures to probabilistic
assumptions. They place a range of recurrent architectures within a common
probabilistic framework, distinguishing changes to the underlying model from
approximations to inference. For the Bayesian Layer, our geometric
analysis shows how repeated evidence can reshape covariance to reduce confusion
between nearby keys. Motivated by this analysis, we examined
recall across controlled key geometries, small language models, and distillation
into a pretrained GDN. These experiments show improved
recall under greater key overlap and more distractor writes than seen during
training, higher WikiText-103 associative recall over corresponding GDN and
RWKV-7 baselines, and improved RULER retrieval over continued GDN after
distillation while retaining similar overall perplexity.
Together, the analysis and experiments demonstrate how probabilistic modeling
can guide the design of recurrent sequence architectures that use uncertainty
geometry to integrate new evidence with stored information.

\subsection*{AI use statement}
We used generative AI tools to assist with manuscript editing, reviewing
mathematical exposition and experimental descriptions, and modifying plotting
code and scientific figures. We reviewed the AI-assisted code, checked relevant
computations through numerical tests, and consulted related literature to assess
the resulting claims and exposition. We take responsibility for the final
content of this paper, including all AI-assisted text, claims, and artifacts.

\subsection*{Reproducibility statement}
The design models and assumptions are specified in
\Cref{sec:framework,sec:specializations}.
\Cref{app:architecture_connections,app:bl_derivation} derive the filtering
updates and architectural connections, \Cref{app:chunk_parallel_filtering}
describes the Bayesian Layer implementation, and
\Cref{app:key_confusion,app:collision_analysis} provide the covariance and
interference analyses. \Crefrange{app:collision_analysis}{app:slimpajama_details}
document the experimental protocols, including data preparation, model
configurations, optimization and tuning settings, training budgets, evaluation
procedures, and how results are summarized across training seeds. Each model
configuration trained for the reported comparisons is evaluated across at least
three training seeds. The supplementary material provides a minimal code
package for reproducing the reported experiments.

\ificlrfinal
\section*{Acknowledgments}
M.D., H.J., and I.M.P.\ were supported by NIH RF1-DA056404 and by the
Portuguese Recovery and Resilience Plan (PRR) through project 62
(Center for Responsible AI), and by Portuguese national funds through
FCT (Funda\c{c}\~ao para a Ci\^encia e a Tecnologia) in the context of
project UIDB/04443/2020.  M.D.\ and C.S.\ were supported by NIH NINDS
awards R01NS127122 and 1RF1NS127122-01.

\paragraph{Competing interests.}
I.M.P.\ and M.D.\ are co-founders of RyvivyR Inc.  RyvivyR had no
role in the design, analysis, or writing of this manuscript.

\fi

\bibliographystyle{iclr2027_conference}
\bibliography{ref}

\appendix

\section{Connections to Recurrent Architectures}
\label{app:architecture_connections}

We derive the two families summarized in \Cref{tab:connections}.
Resetting the Bayesian Layer's predicted covariance gives Delta-rule
updates. A separate latent-input model gives additive updates by
treating each observation as evidence about an independent write
coefficient. Throughout, input-dependent model parameters are treated
as supplied quantities.

\subsection{Covariance reset and Delta-rule updates}
\label{app:covariance_reset}

In the Bayesian Layer recurrence, replace the predicted covariance by
$\bar\vP_t=\lambda_t\vI$, with $\lambda_t>0$. The gain becomes
\[
  \beta_t\vu_t
  =\frac{\lambda_t}{r_t^2+\lambda_t\|\vk_t\|^2}\,\vk_t
  =\eta_t\vk_t.
\]
Writing $\bar\vM_t=\vA_t\vM_{t-1}$, substitution into the mean update gives
\begin{equation}
  \vM_t=(\vI-\eta_t\vk_t\vk_t^\top)\bar\vM_t
         +\eta_t\vk_t\vv_t^\top.
  \label{eq:app_reset_update}
\end{equation}
This proves the isotropic-reset reduction in \Cref{prop:isotropic_reset}.
The resulting Delta-rule update~\citep{schlag2021linear,yang2025gated}
retains error correction while replacing accumulated covariance with
a prescribed isotropic covariance at every step. Choices of $\vA_t$
give the identity, scalar-decay, and diagonal-decay variants in
\Cref{tab:connections}.
Allowing column-specific reset scales and observation variances gives
the Longhorn update shown in \Cref{tab:connections}. Longhorn additionally
diagonalizes the rank-one gate in its implemented
recurrence~\citep{liu2025longhorn}.

\subsection{Latent-input model and additive updates}
\label{app:latent_input}

Under the latent-input model of \Cref{def:latent_input_model}, the
current write coefficient $\vphi_t$ is independent of the previous
memory, and $\vv_t$ observes only $\vphi_t$. Marginalizing the write
coefficient gives the design factor
\begin{equation}
  \Psi_t(\vS_{t-1},\vS_t;\vx_t)
  =\int
    \delta_{\vA_t\vS_{t-1}+\vk_t\vphi^\top}(\vS_t)\,
    \mathcal N(\vphi;\mathbf{0},\lambda_t\vI_m)\,
    \mathcal N(\vv_t;\vphi,r_t^2\vI_m)\,d\vphi,
  \label{eq:li_factor}
\end{equation}
where $\delta_{\vA_t\vS_{t-1}+\vk_t\vphi^\top}$ denotes the point-mass
transition kernel at $\vS_t=\vA_t\vS_{t-1}+\vk_t\vphi^\top$.
Together with the initial distribution of $\vS_0$, this factor defines the
distribution in \eqref{eq:design_prefix}.
Substitution into the filtering recursion \eqref{eq:design_filter}
shows that $\mu_t$ is obtained by applying these dynamics to the
joint belief
\[
  \mu_{t-1}(\vS_{t-1})\,p_t(\vphi_t\mid\vv_t).
\]
Thus the observation updates the write coefficient while leaving
the belief over $\vS_{t-1}$ unchanged. Under this design model, Gaussian
conditioning~\citep[Lemma~A.2]{sarkka2013bayesian} gives
\[
  \E[\vphi_t\mid\vv_t]=\omega_t\vv_t,
  \qquad
  \omega_t=\frac{\lambda_t}{\lambda_t+r_t^2}.
\]
Taking the expectation of the memory dynamics under this joint belief
yields
\begin{equation}
  \vM_t=\E_{\mu_t}[\vS_t]
       =\vA_t\vM_{t-1}+\omega_t\vk_t\vv_t^\top,
  \label{eq:li_exact}
\end{equation}
establishing \Cref{prop:additive_update}.

Defining $\tilde\vv_t=\omega_t\vv_t$ gives the additive-write form
\[
  \vM_t=\vA_t\vM_{t-1}+\vk_t\tilde\vv_t^\top.
\]
For $\vA_t=\vI$, this is the memory update of linear
attention~\citep{Katharopoulos2020-cz}. The unit write coefficient can
therefore be obtained by absorbing $\omega_t$ into the value map, or
in the noiseless limit $r_t^2\to0$. The linear readout remains
$\vy_t=\vM_t^\top\vq_t$. Its evaluation and subsequent mean updates
require only $\vM_t$, even though uncertainty in the write coefficients
generally leaves nonzero posterior uncertainty over memory.

The correspondences in \Cref{tab:connections} concern these recurrent
memory updates. Surrounding projections, normalization, additional
gates, and feedforward layers require their own specifications.

\paragraph{Distinction from xLSTM's matrix-valued state.}
The mLSTM of \citet{beck2024xlstm} uses forget and input gates
$f_t,i_t$ in the content update
\[
  \vC_t=f_t\vC_{t-1}+i_t\vv_t\vk_t^\top
\]
and reads from $\vC_t$ directly. This outer-product state stores values
and need not be symmetric or positive semidefinite. In the Bayesian
Layer, $\vM_t$ stores the posterior mean, while the separate posterior
covariance $\vP_t$ determines the gain and gate. The readout uses only
$\vM_t$.

\section{Derivation of the Bayesian Layer Recurrence}
\label{app:bl_derivation}

We derive \Cref{prop:update} and show that filtering preserves
independent Gaussian memory columns with a shared covariance. At each
step, the input-dependent model parameters are treated as supplied
quantities. The dynamics and observation likelihood in
\Cref{def:design_model} factor across columns, so it suffices to derive
the update for one column.

\subsection{Filtering recurrence and closure}
\label{app:error_correction}
\label{app:kronecker_closure}

\paragraph{Prediction.}
For $t\geq1$, suppose $\mu_{t-1}$ has the belief structure
\eqref{eq:belief_form}. The dynamics preserve independent Gaussian
columns, with predicted distribution
\[
  \vS_t^{(j)}\sim\mathcal N(\bar\vM_t^{(j)},\bar\vP_t),
  \qquad
  \bar\vM_t=\vA_t\vM_{t-1},
  \qquad
  \bar\vP_t=\vA_t\vP_{t-1}\vA_t^\top+\ell_t^2\vI.
\]

\paragraph{Conditioning.}
Under the predicted design model, column $j$ and its scalar observation
have joint distribution
\[
  \begin{bmatrix}\vS_t^{(j)}\\v_{t,j}\end{bmatrix}
  \sim\mathcal N\!\left(
    \begin{bmatrix}\bar\vM_t^{(j)}\\\vk_t^\top\bar\vM_t^{(j)}\end{bmatrix},
    \begin{bmatrix}
      \bar\vP_t & \bar\vP_t\vk_t\\
      \vk_t^\top\bar\vP_t & r_t^2+\vk_t^\top\bar\vP_t\vk_t
    \end{bmatrix}
  \right).
\]
Define $\vu_t=\bar\vP_t\vk_t$ and
$\beta_t=(r_t^2+\vk_t^\top\vu_t)^{-1}$; the denominator is positive
because $r_t>0$. Gaussian conditioning gives
$\vM_t^{(j)}=\bar\vM_t^{(j)}+\beta_t\vu_t
(v_{t,j}-\vk_t^\top\bar\vM_t^{(j)})$.
Stacking these means and writing the posterior covariance gives
\begin{equation}
  \vM_t=(\vI-\beta_t\vu_t\vk_t^\top)\bar\vM_t
         +\beta_t\vu_t\vv_t^\top,
  \qquad
  \vP_t=\bar\vP_t-\beta_t\vu_t\vu_t^\top.
  \label{eq:error_correction}
\end{equation}
These are the updates in \Cref{prop:update}, with column-wise Kalman
gain $\vK_t=\beta_t\vu_t$.

\paragraph{Closure and initialization.}
Both the predicted distribution and likelihood factor across columns,
so conditioning preserves independence. The posterior covariance is
identical for every column, giving
\[
  \operatorname{Cov}_{\mu_t}[\vecOp(\vS_t)]=\vI_m\otimes\vP_t.
\]
The initial distribution in \Cref{def:design_model} has this structure with
$\vM_0=\mathbf{0}$ and $\vP_0=\sigma_0^2\vI_D$. Prediction and conditioning
therefore apply from $t=1$, establishing the belief structure and recurrence
for every subsequent step.
Each column retains a dense $D\times D$ covariance, shared exactly under
the design-model assumptions. The implemented recurrent state is
$\vh_t=(\vM_t,\vP_t)$, while the readout is $\vM_t^\top\vq_t$.

The same derivation applies when the independent process-noise columns
share any covariance $\vQ_t\succeq0$, replacing $\ell_t^2\vI$ by $\vQ_t$
in covariance prediction. \Cref{app:kernel_dynamics} uses this extension
for rotational dynamics.

\section{Filtering implementation}
\label{app:chunk_parallel_filtering}

We implement the Bayesian Layer using chunks of length $L$. Each chunk
receives the posterior mean and covariance from the preceding chunk,
computes its outputs, and returns the updated belief parameters.
The implementation reorganizes the recurrence from
\Cref{app:bl_derivation} around batched matrix products, compact
histories of update directions, and tiled GPU operations.
Chunk-boundary state propagation and tiled evaluation also underpin
efficient GLA and DeltaNet
implementations~\citep{yang2024gated,yang2024parallelizing}.
The equations below describe our covariance and mean computations.

\subsection{Structured dynamics and noise}
\label{app:kernel_dynamics}

We use diagonal transitions or blocks of damped rotations,
\[
  \vA_t^{(j)}
  =\rho_{t,j}
  \begin{bmatrix}
    \cos\theta_{t,j}&-\sin\theta_{t,j}\\
    \sin\theta_{t,j}&\cos\theta_{t,j}
  \end{bmatrix},
  \qquad 0<\rho_{t,j}\leq1.
\]
Diagonal entries likewise have magnitude at most one.
These structures allow multiplication by $\vA_t$ through
coordinate-wise scaling or transformations within each pair.
Throughout this section, $D$ counts real state coordinates, including
both coordinates of each rotational block.

For rotational blocks, keys and queries address the real coordinate:
their embedded forms are $(k_{t,j},0)$ and $(q_{t,j},0)$.
Below, $\vk_t,\vq_t\in\reals^D$ denote these embedded vectors.
We inject process noise into the same coordinates, giving the block
covariance
\[
  \vQ_t^{(j)}
  =\ell_t^2
  \begin{bmatrix}
    1&0\\
    0&0
  \end{bmatrix}.
\]
Rotations subsequently transport this uncertainty between both
coordinates. For ordinary real diagonal dynamics, we use
$\vQ_t=\ell_t^2\vI$, as in the main-text design model.
Noise variances are parameterized using softplus with a positive floor.

We use one covariance per head, shared across its value columns, and
present the chunk equations with the head index suppressed.

\subsection{Exact chunk computation}
\label{app:kernel_chunk}

Number the steps within a chunk $1,\ldots,L$, and denote its entry
belief parameters by $(\vM_0,\vP_0)$. Define the forward transition
products
\[
  \Phi_{t\leftarrow s}
  =\vA_t\vA_{t-1}\cdots\vA_{s+1},
  \qquad
  \Phi_{s\leftarrow s}=\vI.
\]
Before incorporating observations from the current chunk, the state
covariance at step $t$ is
\begin{equation}
  \boldsymbol\Pi_t
  =
  \Phi_{t\leftarrow0}\vP_0\Phi_{t\leftarrow0}^{\top}
  +
  \sum_{s=1}^{t}
    \Phi_{t\leftarrow s}\vQ_s\Phi_{t\leftarrow s}^{\top}.
  \label{eq:kernel_unconditioned_cov}
\end{equation}
The first term propagates the entry covariance; the second accumulates
process noise. All input-dependent model parameters are treated as
supplied quantities.

\paragraph{Gain computation.}
Each earlier observation contributes a rank-one covariance downdate.
Transporting its update direction to step $t$,
\[
  \vw_{t,s}=\Phi_{t\leftarrow s}\vu_s,
\]
and unrolling the Kalman covariance
recurrence~\citep[Theorem~4.2]{sarkka2013bayesian} gives
\begin{equation}
  \bar\vP_t
  =
  \boldsymbol\Pi_t-\sum_{s<t}\beta_s\vw_{t,s}\vw_{t,s}^{\top}.
  \label{eq:kernel_cov_history}
\end{equation}
Consequently, the next update direction and gain scale can be computed as
\begin{equation}
  \vu_t
  =
  \boldsymbol\Pi_t\vk_t
  -
  \sum_{s<t}
    \beta_s\vw_{t,s}(\vw_{t,s}^{\top}\vk_t),
  \qquad
  \beta_t=(r_t^2+\vk_t^\top\vu_t)^{-1}.
  \label{eq:kernel_gain_history}
\end{equation}

This expression avoids constructing a dense covariance at every
position. The entry-covariance contributions to
$\boldsymbol\Pi_t\vk_t$ are evaluated using a batched multiplication
of $\vP_0$ by transformed keys, followed by the corresponding forward
transforms. The process-noise contributions use the diagonal or
rotational block structure. A sequential loop within the chunk then
computes $(\vu_t,\beta_t)$, retaining the preceding update directions
and gain scales.

\paragraph{Mean update and readout.}
Once these gains are available, the mean update is
\begin{equation}
  \bar\vM_t=\vA_t\vM_{t-1},
  \qquad
  \vM_t
  =
  (\vI-\beta_t\vu_t\vk_t^\top)\bar\vM_t
  +\beta_t\vu_t\vv_t^\top.
  \label{eq:eff_M_affine}
\end{equation}
The output uses the posterior mean after incorporating the current
observation:
\[
  \vy_t=\vM_t^\top\vq_t.
\]
The kernel processes tiles of value columns, carrying each mean tile
through the chunk and emitting its outputs. All value columns within a
head reuse the same gain sequence.

\paragraph{Exit covariance.}
At the final position, collect the transported update directions into
\[
  \vW=
  \begin{bmatrix}
    \Phi_{L\leftarrow1}\vu_1&
    \cdots&
    \Phi_{L\leftarrow L}\vu_L
  \end{bmatrix}.
\]
The covariance passed to the next chunk is
\begin{equation}
  \vP_L
  =
  \boldsymbol\Pi_L
  -\vW\,\diag(\beta_1,\ldots,\beta_L)\,\vW^\top.
  \label{eq:kernel_exit_cov}
\end{equation}
The accumulated downdates are therefore evaluated as a tiled matrix
product. This representation preserves the dense covariance: the
compact history records its exact within-chunk corrections.

Chunk boundaries are processed sequentially. Within each chunk,
parallelism is distributed across sequences, heads, and state or value
tiles; the gain and mean passes retain
sequential dependencies over the chunk's positions.

\subsection{Numerical treatment, backward pass, and cost}
\label{app:kernel_cost}

The implementation uses forward transition products throughout.
For a rotational block,
\[
  \Phi_{t\leftarrow s}^{(j)}
  =
  \exp\!\left(\sum_{a=s+1}^{t}\log\rho_{a,j}\right)
  R\!\left(\sum_{a=s+1}^{t}\theta_{a,j}\right),
\]
where $R(\theta)$ is the rotation matrix above. Relative products are
evaluated from differences of cumulative log-radii and angles, without
explicitly forming inverse transition products. Under the stated
parameterization, $\|\Phi_{t\leftarrow s}\|_2\leq1$.

We implement the tiled computations in
Triton~\citep{tillet2019triton}. Recurrent quantities and chunk-boundary
states are maintained in fp32. The chunk identities above are exact
in real arithmetic; floating-point execution additionally floors
innovation denominators at $10^{-6}$.

A custom backward pass differentiates the entry-covariance products,
gain computation, mean/readout recurrence, and exit-covariance update.
Training retains chunk-boundary states and compact intermediate
histories for backward. Optional checkpointing recomputes these
histories and reconstructs intermediate boundaries by replaying
chunks from saved states.

For one head with $m$ value coordinates,
recurrent-state storage is $\mathcal O(Dm+D^2)$.
The entry and exit covariance computations cost
$\mathcal O(LD^2)$ per chunk, the gain-history computation costs
$\mathcal O(L^2D)$, and the mean updates and readouts cost
$\mathcal O(LDm)$. Over a sequence of length $T$, total forward work
is therefore
\[
  \mathcal O\!\left(T\left[D^2+LD+Dm\right]\right).
\]
For fixed chunk length, this is linear in sequence length.
Training activation storage additionally depends on the saved chunk
boundaries and histories, and can be reduced through checkpointing.

\section{Memory resolution and covariance geometry}
\label{app:key_confusion}

This section extends the geometry in \Cref{sec:anatomy}. We derive how stored
evidence sets angular resolution and how covariance geometry controls interference. The calculations and figures use
fixed keys and prescribed noise variances, without model training.

\subsection{Setup and confusion}
For the analysis below, we use identity dynamics $\vA_t=\vI$, real-valued states,
unit keys, and constant variances $\ell_t^2=\ell^2$ and $r_t^2=r^2$.
The memory mean is $\vM_t\in\mathbb{R}^{D\times m}$ and the covariance
$\vP_t\in\mathbb{R}^{D\times D}$ is shared by the $m$ value columns.
The prior is $\vM_0=\mathbf{0}$ and $\vP_0=\sigma_0^2\vI$ with $\sigma_0^2=3$,
the value used in the numerical examples. Analytical results retain general $\sigma_0^2>0$.
As in the main text, $\rho=\vk_A^\top\vk_B$ is the key overlap.
We also use the angle $\theta=\arccos\rho$ to describe angular resolution.
We take $0<\rho<1$; endpoint formulas are interpreted by continuity.
The repeated-key analysis in \Cref{app:gain_envelopes} permits a general incoming
covariance; the exact eigenvalue formulas here use an isotropic prior.

\subsubsection{Cross-key interference and confusion}

Let $\Delta_{B\mid A}$ be the readout change
at $\vk_B$ caused by a write of $\vk_A$. With identity dynamics, the predicted
mean is $\bar{\vM}_t=\vM_{t-1}$ and the predicted covariance is
$\bar{\vP}_t=\vP_{t-1}+\ell^2\vI$. Define
$\beta_t=(r^2+\vk_A^{\top}\bar{\vP}_t\vk_A)^{-1}$. Then
\begin{equation}
\begin{aligned}
\Delta_{B\mid A}
  &:= \hat{\vv}_t(\vk_B)-\bar{\vv}_t(\vk_B)\\
  &= \beta_t\,\vk_B^{\top}\bar{\vP}_t\vk_A\,
     \bigl(\vv_A-\bar{\vv}_t(\vk_A)\bigr).
\end{aligned}
\label{ks:eq:cross_update}
\end{equation}
This is the cross-key interference of \Cref{eq:cross_key_readout}.

Define $\Delta_{A\mid A}$ as the readout change at $\vk_A$ after the same
observation. For $\Delta_{A\mid A}\ne\mathbf{0}$, their shared value error
and scalar gain cancel, giving the confusion defined in
\Cref{eq:cross_key_change_ratio}:
\begin{equation}
  \mathrm{confusion}
  :=\frac{\|\Delta_{B\mid A}\|}{\|\Delta_{A\mid A}\|}
  =\frac{|\vk_B^\top\bar{\vP}_t\vk_A|}
        {\vk_A^\top\bar{\vP}_t\vk_A}.
  \label{ks:eq:confusion}
\end{equation}
The covariance expression remains defined when the value error is zero,
although the ratio of readout changes does not. The figures evaluate this
expression. For the repeated-key and orthogonal-gap histories below, its
numerator before taking the absolute value is nonnegative.

With isotropic covariance, confusion equals $\rho$. The Delta-rule and
additive updates considered in the main text have the same ratio for any
nonzero write along the incoming key. Low confusion means that the readout
at $\vk_B$ changes little relative to the readout at $\vk_A$; confusion near
one means that the two readout changes have similar magnitudes.

\subsection{Confusion after repeated writes}
\label{ks:sec:one}
\label{ks:sec:stored-noise}

\begin{proposition}[Confusion after repeated writes]
\label{ks:prop:one}
Assume identity dynamics and prior covariance $\vP_0=\sigma_0^2\vI$ with
$\sigma_0^2>0$. Let $\vk_B$ and $\vk_A$ be unit keys with
$\vk_B^{\top}\vk_A=\rho$. Let $\vk_\perp$ be the unit vector in their
plane perpendicular to $\vk_B$, so
$\vk_A=\rho\vk_B+\sqrt{1-\rho^2}\,\vk_\perp$. Write $\vk_B$ $n$ times, then $\vk_A$ once, using
process variance $\ell^2\ge0$ and observation variance $r^2>0$ at every write.

The covariance eigenvectors along and orthogonal to $\vk_B$ stay fixed.
Let $\sigma_B^2$ and $\sigma_\perp^2$ denote the posterior variance of each
coordinate of the memory readout at $\vk_B$ and $\vk_\perp$, respectively.
Starting from $\sigma_B^2=\sigma_\perp^2=\sigma_0^2$, each $\vk_B$ write gives
\begin{equation}
\sigma_\perp^2\leftarrow \sigma_\perp^2+\ell^2,\qquad
\sigma_B^2\leftarrow
\left(\frac{1}{\sigma_B^2+\ell^2}+\frac{1}{r^2}\right)^{-1}.
\label{ks:eq:recursion}
\end{equation}
With $\sigma_B^2,\sigma_\perp^2$ denoting the values after these $n$ writes,
the following $\vk_A$ write has confusion
\begin{equation}
\mathrm{confusion}=\frac{\rho}{1+(\kappa-1)(1-\rho^2)},\qquad
\kappa=\frac{\sigma_\perp^2+\ell^2}{\sigma_B^2+\ell^2}.
\label{ks:eq:accumulation}
\end{equation}
Here $\bar\sigma_B^2=\sigma_B^2+\ell^2$ and
$\bar\sigma_\perp^2=\sigma_\perp^2+\ell^2$ are the predicted readout variances
before the $\vk_A$ observation, so $\kappa=\bar\sigma_\perp^2/\bar\sigma_B^2$.
\end{proposition}

\begin{proof}
\textbf{Prediction and precision gain.}
Prediction adds $\ell^2$ to every covariance eigenvalue without changing
the eigenvectors. Conditioning on $\vk_B$ then adds $\vk_B\vk_B^{\top}/r^2$ to the
predicted precision. This adds $1/r^2$ along $\vk_B$ and zero in each orthogonal
direction, giving Eq.~\eqref{ks:eq:recursion}. Thus $\sigma_\perp^2=\sigma_0^2+n\ell^2$,
while $\sigma_B^2<r^2$ after each write.

\textbf{Confusion.}
Write $\vk_A=\rho\vk_B+\sqrt{1-\rho^2}\,\vk_\perp$, where $\vk_\perp$
is a unit vector orthogonal to $\vk_B$. Before conditioning on $\vk_A$,
\[
\bar{\vP}_t\vk_A=\bar\sigma_B^2\rho\vk_B
       +\bar\sigma_\perp^2\sqrt{1-\rho^2}\,\vk_\perp.
\]
Substituting into Eq.~\eqref{ks:eq:confusion} and dividing numerator and
denominator by $\sigma_B^2+\ell^2$ gives Eq.~\eqref{ks:eq:accumulation}.
\end{proof}

\paragraph{No process noise.}
When $\ell^2=0$, precision gains accumulate directly:
\begin{equation}
\vP_n^{-1}=\sigma_0^{-2}\vI+\frac{n}{r^2}\vk_B\vk_B^{\top}.
\label{ks:eq:precision}
\end{equation}
The posterior variances are therefore
\begin{equation}
\sigma_B^2=\left(\sigma_0^{-2}+\frac{n}{r^2}\right)^{-1},\qquad
\sigma_\perp^2=\sigma_0^2.
\label{ks:eq:eigvals}
\end{equation}
Hence $\kappa=1+n\sigma_0^2/r^2$, and Eq.~\eqref{ks:eq:accumulation} reduces to
\begin{equation}
\mathrm{confusion}=\frac{\rho}{1+(n\sigma_0^2/r^2)(1-\rho^2)}.
\label{ks:eq:one0}
\end{equation}
\Cref{ks:fig:one}a plots this zero-noise case for different amounts of
stored evidence.

\paragraph{Resolution angle without process noise.}
For small angles, $\sin\theta\approx\theta$ and $\cos\theta\approx 1$, and
Eq.~\eqref{ks:eq:one0} becomes
\begin{equation}
\mathrm{confusion}\approx \frac{1}{1+(\theta/\theta^\ast)^2},\qquad
\theta^{\ast 2} = \frac{1}{\kappa-1} = \frac{r^2}{n\sigma_0^2}.
\label{ks:eq:width0}
\end{equation}
The resolution angle marks a smooth crossover: closer keys have more strongly
coupled updates, while farther keys have weaker cross-key interference. It
is not a hard boundary between discrete memories.

\paragraph{With process noise.}
For $\ell^2>0$, each write adds uncertainty in all directions, then reduces
it along $\vk_B$. After $n\ge1$ writes, \Cref{ks:eq:recursion} gives
\begin{equation}
\begin{aligned}
\sigma_\perp^2&=\sigma_0^2+n\ell^2,\qquad 0<\sigma_B^2<r^2,\\
\kappa&=\frac{\sigma_0^2+(n+1)\ell^2}{\sigma_B^2+\ell^2}
>\frac{\sigma_0^2+(n+1)\ell^2}{r^2+\ell^2}.
\end{aligned}
\label{ks:eq:noise_contrast}
\end{equation}
The extra $\ell^2$ in $\kappa$ comes from the final $\vk_A$ write.
At fixed $\ell^2>0$ and $r^2$, repeated writes still build directional
contrast: uncertainty orthogonal to the stored key grows without bound,
while uncertainty along it remains bounded by $r^2$.
\Cref{ks:fig:one}b instead fixes $n=16$ and $r^2=0.048$ and varies $\ell^2$;
in this setup, increasing process noise broadens the angular range of coupled
updates.

\begin{figure}[t]
\centering
\includegraphics[width=\linewidth]{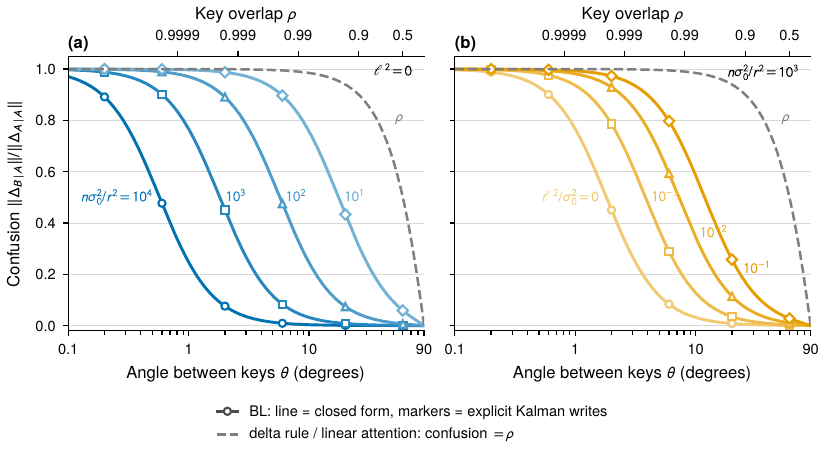}
\caption{Accumulated evidence and process noise control the angular range of
coupled updates. Confusion is the ratio of readout-change magnitudes
(\Cref{ks:eq:confusion}), with identity dynamics and $\sigma_0^2=3$.
(a)~No process noise: vary the accumulated evidence $n\sigma_0^2/r^2$;
lines show \Cref{ks:eq:one0}.
(b)~Fixed evidence: $n=16$ and $r^2=0.048$, so $n\sigma_0^2/r^2=10^3$.
Vary $\ell^2/\sigma_0^2$, adding $\ell^2\vI$ before every update during
the $n$ writes of $\vk_B$ and the final write of $\vk_A$.
Lines show \Cref{ks:eq:accumulation} with eigenvalues from \Cref{ks:eq:recursion}.
Markers show explicit Kalman updates; dashed curves show the Delta-rule and
linear-attention reference $\mathrm{confusion}=\rho$. These are deterministic update
calculations, without model training.}
\label{ks:fig:one}
\end{figure}

\subsection{Covariance geometry under repeated writes}
\label{app:cross_key_balance}
\label{app:gain_envelopes}

We now allow nonzero cross-covariance from earlier writes, rather than
requiring the repeated key to remain a covariance eigenvector. This extends
\Cref{ks:sec:one} to the histories used in the recall experiment of
\Cref{app:collision_details}. The repeated key remains $\vk_B$, followed by an observation at $\vk_A$.
Assume identity dynamics, unit keys, and fixed $r^2,\ell^2>0$.
Write $\vk_A=\rho\vk_B+s\vk_\perp$, where
$s=\sqrt{1-\rho^2}$, $0\leq\rho<1$, and $\vk_\perp$ is a unit vector
perpendicular to $\vk_B$. Before the $n$th consecutive $B$-write,
define the predicted covariance entries
\[
  x_n=\vk_B^\top\bar\vP_n\vk_B,\qquad
  c_n=\vk_\perp^\top\bar\vP_n\vk_B,\qquad
  z_n=\vk_\perp^\top\bar\vP_n\vk_\perp.
\]
Conditioning and prediction give
\begin{align}
  x_{n+1}&=\frac{r^2x_n}{r^2+x_n}+\ell^2,
    \label{eq:crossover_sigma_recursion}\\
  c_{n+1}&=\frac{r^2}{r^2+x_n}\,c_n,
    \label{eq:crossover_coupling}\\
  z_{n+1}&=z_n+\ell^2-\frac{c_n^2}{r^2+x_n}.
    \label{eq:perpendicular_variance_recursion}
\end{align}
Since $x_n\geq\ell^2$, the derivative of the first recurrence is at most
$(r^2/(r^2+\ell^2))^2<1$. Thus $x_n$ converges to
\begin{equation}
  x_*=\frac{\ell^2+\sqrt{\ell^4+4r^2\ell^2}}{2}.
  \label{eq:repeated_key_fixed_point}
\end{equation}
The second recurrence contracts $c_n$ geometrically. The cumulative
subtraction in the third is therefore finite, giving
$z_n=n\ell^2+O(1)$. Thus the predicted readout variance at
$\vk_B$ approaches a finite limit, while that at $\vk_\perp$ grows. The scalar gain
at the repeatedly observed key approaches
\begin{equation}
  g^{\mathrm{ss}}=\frac{x_*}{r^2+x_*}.
  \label{eq:flood_gain_floor}
\end{equation}
Since $\vk_A^\top\bar\vP_n\vk_A=\rho^2x_n+2\rho s c_n+s^2z_n$,
the variance along $\vk_A$ increases asymptotically by
$(1-\rho^2)\ell^2$ per step. Posterior variances have the same limiting
increment because prediction adds $\ell^2$ along every unit direction.

For the first subsequent $A$-write, let $x,c,z$ denote these predicted
covariance entries. Assume no coupling between the two-key plane and other
coordinates, as in the toy protocol of \Cref{app:collision_details}.
The gain vector $\vK_A:=\beta_t\bar{\vP}_t\vk_A$ is
\begin{equation}
  \vK_A=
  \frac{(\rho x+s c)\vk_B+(\rho c+s z)\vk_\perp}
       {r^2+\rho^2x+2\rho s c+s^2z}.
  \label{eq:boundary_gain_vector}
\end{equation}
At fixed $\rho<1$, increasing the number of preceding $B$-writes gives
\begin{equation}
  \vK_A\longrightarrow\frac{\vk_\perp}{s},
  \qquad
  \vk_A^\top\vK_A\longrightarrow1,
  \qquad
  \vk_B^\top\vK_A\longrightarrow0.
  \label{eq:boost_gain_ceiling}
\end{equation}
The $A$-write can therefore correct nearly all of its value error while
barely changing the readout at $\vk_B$. Greater overlap requires more
$B$-writes for this limit to apply; the finite toy schedule retains
cross-key interference.

\paragraph{Update direction and confusion.}
The covariance reduction at a $B$-write is
\[
  \bar\vP_n-\vP_n
  =\frac{\bar\vP_n\vk_B\vk_B^\top\bar\vP_n}
         {r^2+\vk_B^\top\bar\vP_n\vk_B}\succeq0.
\]
This compares the beliefs immediately before and after conditioning;
process noise prevents a claim of monotone shrinkage across time steps.
For the first subsequent $A$-write, the exact direction in the two-key
plane is $(\rho x+s c)\vk_B+(\rho c+s z)\vk_\perp$.
When the cross-covariance terms are negligible, this becomes
\begin{equation}
  \vu_t=\bar{\vP}_t\vk_A
  \approx \rho\,\bar\sigma_B^2\vk_B
  +\sqrt{1-\rho^2}\,\bar\sigma_\perp^2\vk_\perp,
  \label{eq:residual_write_main}
\end{equation}
where $\bar\sigma_B^2=x$ and $\bar\sigma_\perp^2=z$.
For $\Delta_{A\mid A}\ne\mathbf{0}$, the common value error and scalar gain
cancel in the ratio of readout-change magnitudes, giving the confusion
\begin{equation}
  \frac{\|\Delta_{B\mid A}\|}{\|\Delta_{A\mid A}\|}
  =\frac{|\vk_B^\top\bar{\vP}_t\vk_A|}{\vk_A^\top\bar{\vP}_t\vk_A}
  =\frac{|\rho x+s c|}{\rho^2x+2\rho s c+s^2z}
  \approx\frac{\rho\,\bar\sigma_B^2}
              {\rho^2\bar\sigma_B^2+(1-\rho^2)\,\bar\sigma_\perp^2}.
  \label{eq:cross_key_change_ratio_derivation}
\end{equation}
The denominator is positive because $\bar\vP_t\succ0$.
With isotropic covariance, the direction remains parallel to $\vk_A$,
and confusion equals $\rho$. In the approximately diagonal regime, a
sufficient condition for low confusion is
$\rho\,\bar\sigma_B^2\ll(1-\rho^2)\bar\sigma_\perp^2$, or
$\bar\sigma_\perp^2/\bar\sigma_B^2\gg\rho/(1-\rho^2)$.
Thus greater overlap requires a larger variance contrast. For fixed
$\rho<1$, the repeated-write recurrences above give $c\to0$, bounded $x$,
and growing $z$, so confusion tends to zero and the update direction
approaches the direction perpendicular to $\vk_B$.

\section{Associative Recall with Overlapping Keys}
\label{app:collision_analysis}
\label{app:closed_forms}
\label{app:collision_diagnostics}

We specify the toy in \Cref{fig:controlled_collision}a and derive its
readout-error behavior using the covariance analysis in
\Cref{app:gain_envelopes}. We then describe the learned controlled-recall
and random-key experiments.

\Cref{tab:consolidated_settings} summarizes optimization, tuning, and
repetition settings across the experiments; individual protocols follow
here and in \Cref{app:mqar_wikitext,app:slimpajama_details}.

\begin{table*}[!htbp]
  \centering
  \small
  \setlength{\tabcolsep}{3.5pt}
  \renewcommand{\arraystretch}{1.12}
  \caption{\textbf{Optimization, tuning, budget, and repetition settings.}
  MQAR learning rates are selected by validation query accuracy;
  WikiText-103 settings are detailed in \Cref{app:from_scratch_wikitext}.
  Both experiments report means and sample standard deviations over
  three training seeds.}
  \label{tab:consolidated_settings}
  \begin{tabular}{@{}p{0.21\linewidth}p{0.76\linewidth}@{}}
    \toprule
    \textbf{Experiment} & \textbf{Settings} \\
    \midrule
    Toy recall with overlapping keys
      & \textbf{Optimization:} None; fixed $r^2{=}\ell^2{=}0.05$.
        \par\textbf{Tuning:} None.
        \par\textbf{Budget:} 32 writes per order; 900 order classes.
        \par\textbf{Repetitions:} Exact order mean $\pm1$ std. \\
    \addlinespace[5pt]
    Learned controlled recall
      & \textbf{Optimization:} AdamW; $\beta{=}(0.9,0.999)$; weight decay
        $10^{-4}$; clipping 1; learning rate $3\times10^{-4}$.
        \par\textbf{Tuning:} One shared setting; no per-model tuning.
        \par\textbf{Budget:} Batch 256; 2,500 updates; query-token cross-entropy.
        \par\textbf{Repetitions:} 3 seeds; mean $\pm1$ std. \\
    \addlinespace[5pt]
    MQAR
      & \textbf{Optimization:} AdamW; $\beta{=}(0.9,0.95)$; weight decay 0.1;
        clipping 1; epoch cosine to zero; no warmup.
        \par\textbf{Tuning:} 4-epoch pilots at width 128; learning rate
        $\in\{3\times10^{-4},6\times10^{-4},10^{-3}\}$.
        \par\textbf{Budget:} Batch 256; at most 32 epochs; training lengths $\leq256$.
        \par\textbf{Repetitions:} 3 seeds; mean $\pm1$ std. \\
    \addlinespace[5pt]
    WikiText-103
      & \textbf{Optimization:} AdamW; $\beta{=}(0.9,0.95)$; weight decay 0.01;
        clipping 1; 20M-token warmup; cosine to $0.1\times$ peak.
        \par\textbf{Tuning:} Validation-loss learning-rate pilots for baselines;
        all models use $10^{-3}$.
        \par\textbf{Budget:} Batch 65,536 tokens; 520M tokens; 7,935 updates;
        sequence length 2,048.
        \par\textbf{Repetitions:} 3 seeds/model; paired streams. \\
    \addlinespace[5pt]
    Layerwise BL distillation
      & \textbf{Optimization:} AdamW; $\beta{=}(0.9,0.999)$; weight decay 0;
        clipping 1; learning rate $5\times10^{-4}$; 25 warmup updates;
        cosine to $0.1\times$ peak.
        \par\textbf{Tuning:} Fixed per-layer distillation recipe.
        \par\textbf{Budget:} 100,007,936 tokens/layer in parallel; 1,526 updates.
        \par\textbf{Repetitions:} 3 seeds; terminal checkpoints. \\
    \addlinespace[5pt]
    GDN continuation
      & \textbf{Optimization:} AdamW; $\beta{=}(0.9,0.95)$; weight decay 0;
        clipping 1; learning rate $3\times10^{-4}$; 50 warmup updates;
        cosine to $0.1\times$ peak.
        \par\textbf{Tuning:} Mixers and MLPs at the same four layers;
        next-token cross-entropy.
        \par\textbf{Budget:} 100,007,936 tokens; 763 updates; batch 131,072 tokens.
        \par\textbf{Repetitions:} 3 seeds; terminal checkpoints. \\
    \bottomrule
  \end{tabular}
\end{table*}

\subsection{Toy protocol}
\label{app:collision_details}
\label{app:experiments}

\paragraph{Associations and schedule.}
The toy stores six associations $A$--$F$ with one-hot values in
$\reals^6$ and keys in $\reals^{16}$:
\[
  \vk_A=\ve_1,\qquad
  \vk_B=\rho\ve_1+\sqrt{1-\rho^2}\,\ve_2,\qquad
  (\vk_C,\vk_D,\vk_E,\vk_F)=(\ve_3,\ve_4,\ve_5,\ve_6),
\]
where $\rho=0.98$. Two initial rounds present each association once,
with independent uniform permutations. Ten $B$-writes then precede
ten $A$-writes, giving 32 writes in total. All methods receive the same
sequences; plotted step zero follows the twelve initial writes.

\paragraph{Updates.}
All methods start from $\vM_0=\mathbf{0}$. BL uses
\Cref{prop:update} with $\vA_t=\vI$, $\vP_0=3\vI$, and
$\ell^2=r^2=0.05$. The diagonal variant additionally discards off-diagonal
covariance entries after each update. The other recurrences are
\begin{align*}
  \vM_t&=\vM_{t-1}+\tfrac12\vk_t
    (\vv_t-\vM_{t-1}^{\top}\vk_t)^\top
    &&\text{(Delta rule)},\\
  \vM_t&=\tfrac12\vM_{t-1}+\tfrac12\vk_t\vv_t^\top
    &&\text{(decayed additive update)}.
\end{align*}
On unit keys, the Delta rule equals covariance reset with
$\bar\vP_t=0.05\vI$ and $r^2=0.05$. The additive baseline here includes
decay; the learned linear-attention baseline is undecayed.

\paragraph{Measurement.}
After initialization and each subsequent write, we record
$E_B(t)=\|\vM_t^\top\vk_B-\vv_B\|^2$, summed over value coordinates.
Curves and bands show the mean and one population standard deviation
over initial orderings. Only the positions of $A$ and $B$ matter:
permuting the four orthogonal identities leaves these readouts unchanged.
We therefore average exactly over $(6\cdot5)^2=900$ equally weighted
position pairs. Statistics are computed before clipping values and
band boundaries below $10^{-4}$ to the lowest plotted tick.

\subsection{Readout error under subsequent writes}
\label{app:panel_a_limits}

Using the covariance evolution derived in \Cref{app:gain_envelopes},
let $t_0$ mark the end of the finite $B$-write phase and
$\vP_{\mathrm b}=\vP_{t_0}$ its posterior covariance. After $n$
subsequent $A$-writes, define
\[
  \va_n=\vM_{t_0+n}^{\top}\vk_A,\qquad
  \vb_n=\vM_{t_0+n}^{\top}\vk_B.
\]
With identity dynamics, an error-correcting update with gain $\vK_n$
satisfies
\begin{align}
  \va_n&=\va_{n-1}+(\vk_A^\top\vK_n)(\vv_A-\va_{n-1}),\notag\\
  \vb_n&=\vb_{n-1}+(\vk_B^\top\vK_n)(\vv_A-\va_{n-1}).
  \label{eq:two_key_readout_recursion}
\end{align}
The main figure evaluates $E_B(t_0+n)=\|\vb_n-\vv_B\|^2$ for ten
$A$-writes. The following limits describe indefinitely repeated writes
from each method's actual boundary state.

\paragraph{Bayesian Layer.}
For dense covariance and fixed $r^2,\ell^2>0$,
\begin{equation}
  \begin{aligned}
    \vb_\infty&=\vb_0+\chi_{\mathrm{BL}}(\vv_A-\va_0),\\
    \chi_{\mathrm{BL}}
      &=\frac{\vk_B^\top(\vP_{\mathrm b}+x_*\vI)\vk_A}
              {\vk_A^\top(\vP_{\mathrm b}+x_*\vI)\vk_A},
  \end{aligned}
  \label{eq:bayes_exact_flood_limit}
\end{equation}
where $x_*$ is given by \eqref{eq:repeated_key_fixed_point}.

\begin{proof}
Let $p_{A,n}=\vk_A^\top\vP_{t_0+n}\vk_A$ and $w_0=1$. The scalar
covariance recurrence and innovation contraction are
\[
  p_{A,n}=\frac{r^2(p_{A,n-1}+\ell^2)}
                   {r^2+p_{A,n-1}+\ell^2},
  \qquad
  w_n=w_{n-1}\frac{r^2}{r^2+p_{A,n-1}+\ell^2},
  \qquad
  \va_n-\vv_A=w_n(\va_0-\vv_A).
\]
Writing $d_n=\vk_B^\top\vP_{t_0+n}\vk_A-\rho p_{A,n}$, the covariance
update also gives $d_n=w_nd_0$, since $\vk_B-\rho\vk_A$ is perpendicular
to $\vk_A$. Substitution into \eqref{eq:two_key_readout_recursion} yields
\[
  \vb_n-\vb_0=
  \left[
    \rho(1-w_n)
    +d_0\sum_{j=1}^{n}
      \frac{w_{j-1}^2}{r^2+p_{A,j-1}+\ell^2}
  \right](\vv_A-\va_0).
\]
Using the scalar covariance recurrence and
$x_*^2-\ell^2x_*-r^2\ell^2=0$,
\begin{equation}
  \frac{w_{j-1}^2}{r^2+p_{A,j-1}+\ell^2}
  =
  \frac{w_{j-1}^2}{p_{A,j-1}+x_*}
  -\frac{w_j^2}{p_{A,j}+x_*}.
  \label{eq:readout_limit_telescope}
\end{equation}
Since $0\leq w_n\leq(r^2/(r^2+\ell^2))^n$, the sum telescopes to
$1/(p_{A,0}+x_*)$ in the limit. The coefficient becomes
$\rho+d_0/(p_{A,0}+x_*)=\chi_{\mathrm{BL}}$.
\end{proof}

\paragraph{Other updates.}
For gains parallel to $\vk_A$,
$\vb_n-\vb_0=\rho(\va_n-\va_0)$. The Delta rule uses
$\vK_n=\tfrac12\vk_A$, giving
\begin{equation}
  \vb_\infty=\vb_0+\rho(\vv_A-\va_0).
  \label{eq:delta_exact_flood_limit}
\end{equation}
The diagonal variant has the same limit in this toy: $\vk_A=\ve_1$
keeps its gain parallel to the key, and $\ell^2>0$ bounds the scalar
gain away from zero. Different boundary readouts $\va_0,\vb_0$ account
for different plateaus. For the decayed additive update,
\begin{equation}
  \vb_n=2^{-n}\vb_0+(1-2^{-n})\rho\vv_A,\qquad
  E_B(\infty)=\|\rho\vv_A-\vv_B\|^2=1+\rho^2,
  \label{eq:additive_toy_readout}
\end{equation}
using the distinct one-hot values.

\subsection{Learned controlled recall}
\label{app:synthetic_train_test_task}

\subsubsection{Sequences and evaluation}
\label{app:episode_construction}

Each sequence uses $K=8$ target--distractor pairs $(T_i,D_i)$ in $D=16$ address
dimensions. Here a bar denotes a supplied address, not a predicted quantity:
\[
  \bar\vk_{T_i}=\ve_{2i-1},\qquad
  \bar\vk_{D_i}=\rho_i\ve_{2i-1}+\sqrt{1-\rho_i^2}\,\ve_{2i}.
\]
Within-pair overlap is $\rho_i$; different pairs are orthogonal.
A fresh random permutation $\pi$ of $\{1,\ldots,2K\}$ assigns values
$\vv_{T_i}=\ve_{\pi(2i-1)}$ and $\vv_{D_i}=\ve_{\pi(2i)}$.
This prevents memorizing a fixed address-to-label mapping.
Tokens have the form $[\texttt{type},\bar\vk,\vv]\in\reals^{33}$:
a write supplies its one-hot value, while a query supplies zeros in
the value coordinates. An explicit write mask freezes each mixer's
recurrent state at query positions, so queries read the stored memory
without applying state prediction or write updates.

A sequence first writes each identity once, then writes each target
$n_b=4$ times, each distractor $n_d$ times, and finally queries each
target once. Writes are independently shuffled within each phase,
and query order is random. Training samples
$n_d\in\{1,2,4,8\}$ and $\rho_i\sim\mathrm{Uniform}(0.60,0.80)$.
For \Cref{fig:controlled_collision}b, evaluation fixes $\rho_i=0.98$
and sweeps $n_d\in\{1,2,4,8,16,32,64\}$. For panel (c), it fixes
$n_d=64$ and uses a common overlap for all pairs from
\[
  \rho\in\{0.60,0.70,0.80,0.85,0.875,0.90,0.925,0.95,
             0.96,0.97,0.975,0.98\}.
\]
Each condition uses 2,048 sequences per training seed, or 16,384
target queries. Accuracy is the fraction of queries whose predicted
label matches the assigned value; chance is $1/(2K)=6.25\%$.
Curves and bands report the arithmetic mean and population standard
deviation over three training seeds.

\subsubsection{Backbone and mixer variants}
\label{app:compared_models}

All models use two residual mixer--MLP blocks with pre-norm RMSNorm,
width 64, four heads with key and value dimensions 16, and SwiGLU
MLPs of hidden width 128. There are no positional embeddings.
A learned map $W_{\mathrm{in}}:\reals^{33}\to\reals^{64}$ embeds
each token into the residual stream. Its address also passes unchanged
to every layer and head:
$\vk_t^{(\ell,h)}=\vq_t^{(\ell,h)}=\bar\vk_t$.
Thus learning cannot change the overlap seen by the recurrent update.

Let $\vz_t^{(\ell-1)}$ denote the residual-stream input to layer $\ell$.
Values are
$\vv_t^{(\ell,h)}=W_V^{(\ell,h)}
\mathrm{RMSNorm}(\vz_t^{(\ell-1)})$, with
$W_V^{(\ell,h)}\in\reals^{16\times64}$.
Gates and noise scales use the same normalized input.
Each head maintains $\vM_t^{(\ell,h)}\in\reals^{16\times16}$ and
reads $(\vM_t^{(\ell,h)})^\top\bar\vk_t$ at queries; concatenated head
outputs are projected and mapped to logits over the $2K$ labels.

\Cref{tab:exp2_models} specifies the recurrent cores compared in this
shared backbone. The Delta-rule label in the main figure denotes the
learned DeltaNet-style update; the fixed-reset control appears only
in the appendix diagnostics.
\begin{table}[!ht]
  \centering
  \small
  \caption{\textbf{Controlled-recall mixer variants.}
  All variants use the supplied keys and queries.}
  \label{tab:exp2_models}
  \renewcommand{\arraystretch}{1.15}
  \begin{tabular}{@{}p{0.23\linewidth}p{0.72\linewidth}@{}}
    \toprule
    Mixer & Update specification \\
    \midrule
    Linear attention & Additive update with $\vA_t=\vI$. \\
    GLA-style & Additive update with $\vA_t=\diag(\boldsymbol{\alpha}_t)$. \\
    Delta rule & Error correction with gain $\eta_t\bar\vk_t$. \\
    Covariance reset & The same correction with fixed gain $\tfrac12\bar\vk_t$. \\
    Diagonal covariance & \Cref{prop:update}, discarding off-diagonal entries of $\vP_t$ after each update. \\
    Bayesian Layer & \Cref{prop:update}, retaining dense $\vP_t$. \\
    \bottomrule
  \end{tabular}
\end{table}

Suppressing layer and head indices, let
$\tilde\vz_t=\mathrm{RMSNorm}(\vz_t)$ and let $\sigma$ denote the logistic sigmoid.
GLA learns $\boldsymbol{\alpha}_t=\sigma(W_g\tilde\vz_t+\vb_g)
\in(0,1)^{16}$; the Delta rule learns
$\eta_t=\sigma(\vw_\eta^\top\tilde\vz_t+b_\eta)\in(0,1)$.
The reset control uses $\bar\vP_t=0.05\vI$ and $r^2=0.05$.
BL and the diagonal variant use identity dynamics and one covariance
per head, with learned input-dependent process variance
$\ell_t^2(\tilde\vz_t)>0$, fixed $r^2=0.05$, and initial covariance
$3\vI$ at each new sequence.

\subsubsection{Optimization}
\label{app:ablations}

Training minimizes cross-entropy at query positions. We use AdamW
with learning rate $3\times10^{-4}$, $\beta=(0.9,0.999)$, weight decay
$10^{-4}$, gradient clipping at norm one, batch size 256, and 2,500
updates. All models use this shared setting without per-model tuning.
The random-key experiment below uses the same backbone, optimizer,
budget, and three training seeds, with models trained separately.

\subsection{Random-key recall}
\label{app:random_keys}

This experiment tests interference beyond the engineered pair structure.
Raw addresses are independently sampled from $\mathcal N(\mathbf0,\vI)$
in $\reals^{64}$ and normalized to unit norm; values are one-hot labels
from $\{1,\ldots,64\}$, re-sampled for each sequence.
Each sequence contains $n$ writes followed by eight queries sampled
uniformly without replacement from the written keys. Tokens use the
same type/address/value format, now with 64 address and 64 value
coordinates.

The address pathway bypasses the learned embedding. A shared frozen
matrix $R\in\reals^{16\times64}$ with orthonormal rows maps raw addresses
to the head dimension, followed by normalization:
\[
  \tilde\vk_t=\frac{R\bar\vk_t}{\|R\bar\vk_t\|}.
\]
Every layer and head uses $\tilde\vk_t$ as both key and query.
Overlap is therefore measured in this projected space.
Training uses $n\in\{16,32,48,64\}$; evaluation uses
$n\in\{96,128,192,256\}$, corresponding to $n/D\in\{6,8,12,16\}$
with $D=16$. The ratio $n/D$ indexes the number of associations per
key dimension. \Cref{fig:appendix_random_key_crowding} reports query
cross-entropy and accuracy across memory loads. Accuracy is the fraction
of queries whose highest-scoring label is correct; chance is $1/64$.

\begin{figure}[!htbp]
  \centering
  \includegraphics[width=0.88\linewidth]{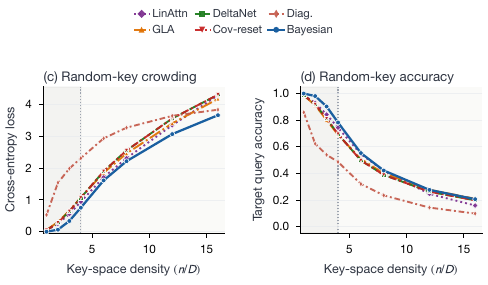}
  \caption{%
    \textbf{Random-key crowding.}
    Query cross-entropy (left) and accuracy (right) versus associations
    per key dimension.
    Models are trained separately with $n/D\leq4$.}
  \label{fig:appendix_random_key_crowding}
\end{figure}
\section{Associative Recall on MQAR and WikiText-103}
\label{app:mqar_wikitext}

\subsection{MQAR protocol and complete results}
\label{app:zoology_details}

\paragraph{Data and evaluation.}
We use the canonical Zoology MQAR generator~\citep{arora2024zoology}
with vocabulary size 8,192. Each sequence presents key--value
associations and later queries their values; training minimizes
cross-entropy at the query targets. The training pool contains
100,000 examples with sequence length $L=64$ and $K=4$ associations,
and 20,000 examples for each of $(L,K)=(128,8)$, $(256,16)$,
$(256,32)$, and $(256,64)$. Evaluation uses 1,000 examples per
setting, with
\[
  (L,K)\in\{(64,4),(64,8),(64,16),(128,32),(256,64),
             (512,128),(1024,256)\}.
\]
We report query accuracy separately for each setting. The final
setting appears in \Cref{fig:mqar_wikitext}a;
\Cref{fig:mqar_complete,tab:mqar_complete} give all seven.

\paragraph{Models.}
All models have two sequence-mixer blocks, tied input and output
embeddings, and width $d\in\{64,128,256\}$, without intervening MLPs.
Head-based mixers use two heads with key and value dimensions $d/2$.
Mamba-2 uses expansion factor one, head and state dimensions $d/2$,
and two groups for its input and output state projections. Thus
width and depth are shared, while the recurrent states and parameter
counts follow each architecture.
The baselines are Gated DeltaNet (GDN)~\citep{yang2025gated},
Gated DeltaNet-2 (GDN-2)~\citep{hatamizadeh2026gateddeltanet2},\footnote{\url{https://github.com/NVlabs/GatedDeltaNet-2}}
MesaNet~\citep{vonoswald2026mesa}, RWKV-7~\citep{peng2025rwkv7},
Palimpsa-D~\citep{bonnet2026learning}, Mamba-2~\citep{dao2024transformers},
and causal attention~\citep{vaswani2017attention} with rotary position
embeddings~\citep{su2024roformer}.
For BL, we use identity state dynamics with learned, input-dependent
process and observation noise. Process variance is initialized to
$0.1$. Its GDN-like mixer applies length-four causal convolutions and
SiLU to the query, key, and value projections, $\ell_2$-normalizes
queries and keys, and applies RMS normalization to the readout.
This MQAR configuration does not use an output gate; the language
model configurations below retain their respective output gates.

\paragraph{Optimization and repetitions.}
We use AdamW with $\beta=(0.9,0.95)$, $\epsilon=10^{-8}$, weight decay
$0.1$, and gradient clipping at norm one. Training batches contain
256 sequences and evaluation batches contain 32. Learning rates decay
to zero with an epoch-level cosine schedule over at most 32 epochs,
without warmup. Training stops early if validation query accuracy
reaches $0.99$, and we retain the checkpoint with the highest
validation query accuracy. Four-epoch pilots at width 128 and seed
123 compare peak learning rates $\{3\times10^{-4},6\times10^{-4},10^{-3}\}$.
We retain $6\times10^{-4}$ for MesaNet and $10^{-3}$ for the remaining
models across widths and seeds. The BL process-noise initialization
uses this same learning rate; its noise-projection weights receive a
$0.25$ learning-rate multiplier. Each reported estimate is the mean
over seeds 123, 234, and 345, with sample standard deviation.

\begin{figure}[p]
  \centering
  \includegraphics[width=\textwidth]{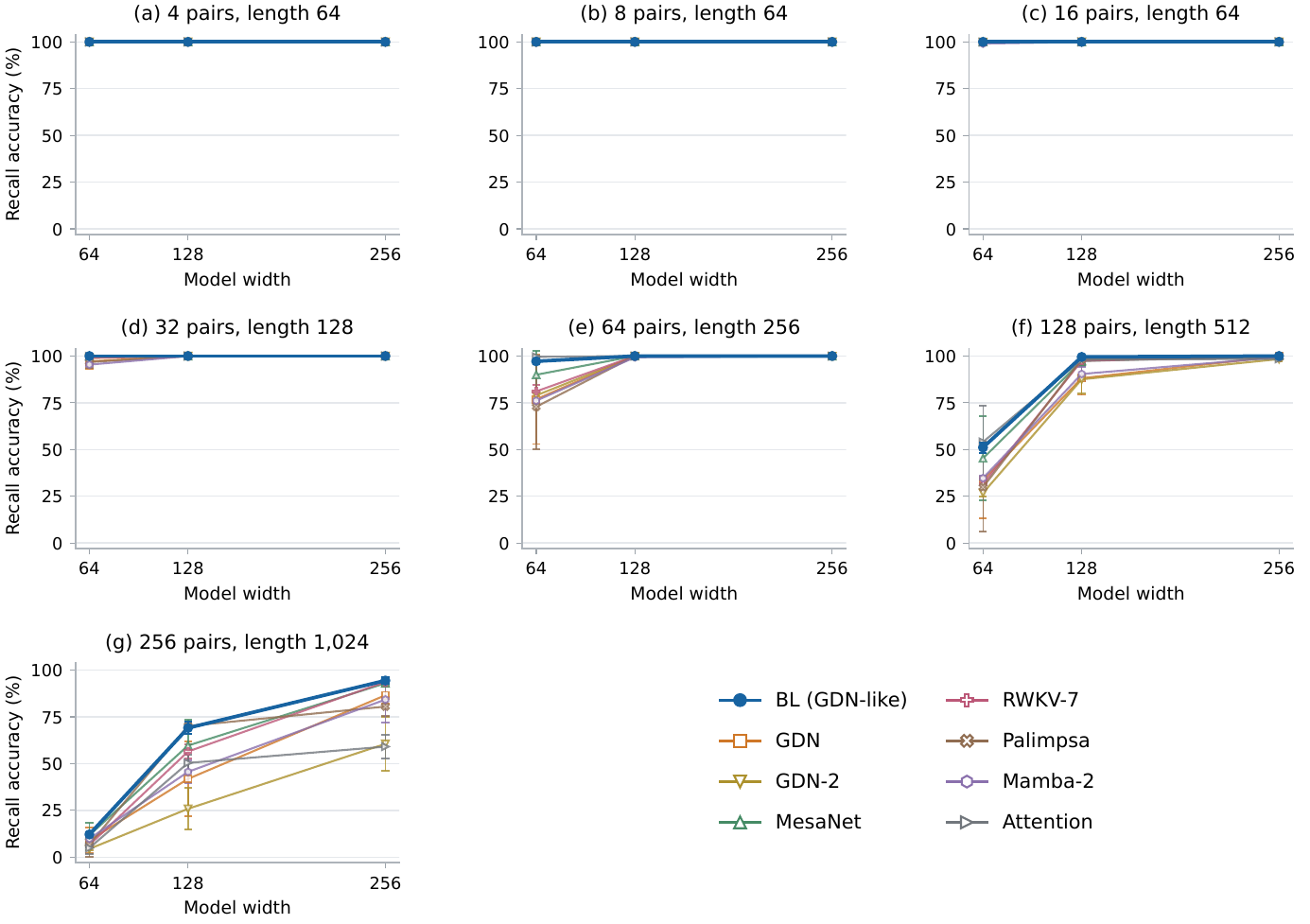}
  \caption{\textbf{Complete MQAR evaluation.} Query accuracy as a
  function of model width for all seven evaluation settings. Titles
  give the number of key--value associations and sequence length.
  Models are trained on sequences of length at most 256. Points show
  means and error bars show one sample standard deviation over three
  training seeds. The final panel also appears in
  \Cref{fig:mqar_wikitext}a.}
  \label{fig:mqar_complete}
\end{figure}

\begin{table}[!htbp]
  \centering
  \scriptsize
  \setlength{\tabcolsep}{1.4pt}
  \renewcommand{\arraystretch}{1.12}
  \caption{\textbf{Complete MQAR results.} Query accuracy (\%) is reported as mean $\pm1$ sample standard deviation over seeds 123, 234, and 345. Column headings give the number of associations $K$ and sequence length $L$.}
  \label{tab:mqar_complete}
  \begin{tabular}{@{}lr*{7}{r}@{}}
    \toprule
    Model & Width & \multicolumn{7}{c}{$(K,L)$} \\
    \cmidrule(l){3-9}
    & & $(4,64)$ & $(8,64)$ & $(16,64)$ & $(32,128)$ & $(64,256)$ & $(128,512)$ & $(256,1024)$ \\
    \midrule
    BL & 64 & $100.00\pm0.00$ & $100.00\pm0.00$ & $100.00\pm0.00$ & $100.00\pm0.00$ & $97.19\pm0.57$ & $51.07\pm2.81$ & $12.14\pm1.11$ \\
     & 128 & $100.00\pm0.00$ & $100.00\pm0.00$ & $100.00\pm0.00$ & $100.00\pm0.00$ & $100.00\pm0.00$ & $99.52\pm0.20$ & $69.09\pm3.30$ \\
     & 256 & $100.00\pm0.00$ & $100.00\pm0.00$ & $100.00\pm0.00$ & $100.00\pm0.00$ & $100.00\pm0.00$ & $100.00\pm0.00$ & $94.43\pm1.70$ \\
    \addlinespace[3pt]
    GDN & 64 & $100.00\pm0.00$ & $99.99\pm0.02$ & $99.90\pm0.12$ & $97.12\pm4.04$ & $76.85\pm23.94$ & $33.68\pm20.53$ & $8.93\pm6.98$ \\
     & 128 & $100.00\pm0.00$ & $100.00\pm0.00$ & $100.00\pm0.00$ & $100.00\pm0.00$ & $99.88\pm0.10$ & $88.14\pm8.93$ & $41.86\pm20.13$ \\
     & 256 & $100.00\pm0.00$ & $100.00\pm0.00$ & $100.00\pm0.00$ & $100.00\pm0.00$ & $100.00\pm0.00$ & $99.85\pm0.15$ & $86.57\pm4.60$ \\
    \addlinespace[3pt]
    GDN-2 & 64 & $100.00\pm0.00$ & $99.99\pm0.01$ & $99.94\pm0.08$ & $99.01\pm0.73$ & $78.92\pm2.48$ & $26.78\pm1.90$ & $4.41\pm0.69$ \\
     & 128 & $100.00\pm0.00$ & $100.00\pm0.00$ & $100.00\pm0.00$ & $99.99\pm0.01$ & $99.91\pm0.07$ & $87.53\pm7.65$ & $25.85\pm11.14$ \\
     & 256 & $100.00\pm0.00$ & $100.00\pm0.00$ & $100.00\pm0.00$ & $99.99\pm0.01$ & $99.93\pm0.11$ & $98.37\pm1.24$ & $60.55\pm14.54$ \\
    \addlinespace[3pt]
    MesaNet & 64 & $100.00\pm0.00$ & $100.00\pm0.00$ & $100.00\pm0.00$ & $99.89\pm0.17$ & $89.97\pm12.90$ & $45.29\pm22.45$ & $11.73\pm6.73$ \\
     & 128 & $100.00\pm0.00$ & $100.00\pm0.00$ & $100.00\pm0.00$ & $100.00\pm0.00$ & $100.00\pm0.00$ & $97.34\pm2.63$ & $59.66\pm13.79$ \\
     & 256 & $100.00\pm0.00$ & $100.00\pm0.00$ & $100.00\pm0.00$ & $100.00\pm0.00$ & $100.00\pm0.00$ & $99.96\pm0.02$ & $92.97\pm0.60$ \\
    \addlinespace[3pt]
    RWKV-7 & 64 & $100.00\pm0.00$ & $100.00\pm0.00$ & $99.98\pm0.01$ & $99.00\pm0.55$ & $81.07\pm3.56$ & $31.57\pm3.12$ & $6.50\pm0.99$ \\
     & 128 & $100.00\pm0.00$ & $100.00\pm0.00$ & $100.00\pm0.00$ & $100.00\pm0.00$ & $100.00\pm0.00$ & $97.41\pm1.64$ & $56.49\pm3.82$ \\
     & 256 & $100.00\pm0.00$ & $100.00\pm0.00$ & $100.00\pm0.00$ & $100.00\pm0.00$ & $100.00\pm0.00$ & $99.98\pm0.02$ & $93.67\pm0.81$ \\
    \addlinespace[3pt]
    Palimpsa & 64 & $99.98\pm0.03$ & $100.00\pm0.01$ & $99.93\pm0.06$ & $96.76\pm3.18$ & $72.87\pm22.82$ & $29.87\pm23.87$ & $6.77\pm6.70$ \\
     & 128 & $100.00\pm0.00$ & $100.00\pm0.00$ & $100.00\pm0.00$ & $100.00\pm0.00$ & $99.99\pm0.02$ & $98.20\pm0.36$ & $70.13\pm1.50$ \\
     & 256 & $100.00\pm0.00$ & $100.00\pm0.00$ & $100.00\pm0.00$ & $100.00\pm0.00$ & $99.99\pm0.01$ & $99.25\pm0.53$ & $80.53\pm4.91$ \\
    \addlinespace[3pt]
    Mamba-2 & 64 & $99.92\pm0.03$ & $99.74\pm0.02$ & $99.09\pm0.27$ & $95.39\pm0.74$ & $76.11\pm1.38$ & $34.67\pm1.28$ & $9.83\pm1.53$ \\
     & 128 & $100.00\pm0.00$ & $99.99\pm0.01$ & $99.98\pm0.04$ & $99.94\pm0.07$ & $99.32\pm0.61$ & $90.39\pm3.96$ & $45.61\pm5.96$ \\
     & 256 & $100.00\pm0.00$ & $100.00\pm0.00$ & $100.00\pm0.00$ & $100.00\pm0.00$ & $99.99\pm0.00$ & $99.36\pm0.76$ & $84.25\pm12.13$ \\
    \addlinespace[3pt]
    Attention & 64 & $99.93\pm0.05$ & $99.94\pm0.05$ & $99.89\pm0.08$ & $99.77\pm0.23$ & $99.57\pm0.47$ & $54.04\pm19.42$ & $4.88\pm3.38$ \\
     & 128 & $99.99\pm0.01$ & $99.98\pm0.03$ & $100.00\pm0.00$ & $99.99\pm0.00$ & $99.99\pm0.01$ & $98.06\pm0.66$ & $50.38\pm4.25$ \\
     & 256 & $99.91\pm0.12$ & $99.95\pm0.03$ & $99.98\pm0.01$ & $99.99\pm0.00$ & $99.97\pm0.03$ & $98.48\pm0.29$ & $59.08\pm6.50$ \\
    \bottomrule
  \end{tabular}
\end{table}

\subsection{WikiText-103 models and evaluation}
\label{app:from_scratch_wikitext}

\paragraph{Shared training protocol.}
We train on the official WikiText-103 training split~\citep{merity2017pointer}
using the 32k-BPE \texttt{fla-hub/gla-1.3B-100B} tokenizer.
Documents begin with a beginning-of-sequence token and are packed
into 2,048-token windows. Models have width 512, four mixer--MLP
blocks, and tied input and output embeddings. Each run processes
520,028,160 tokens in 7,935 updates of 65,536 tokens, with seeds
0, 1, and 2 and paired token streams across models. We use AdamW
with $\beta=(0.9,0.95)$, $\epsilon=10^{-8}$, weight decay $0.01$,
and gradient clipping at norm one. A 20M-token linear warmup precedes
cosine decay to $0.1$ times the peak learning rate. Baseline
learning-rate pilots compare
$\{3\times10^{-4},6\times10^{-4},10^{-3}\}$ for 20M tokens using
validation loss, with a 1M-token pilot warmup. All reported models
use a peak learning rate of $10^{-3}$; variants that change the
surrounding block components retain this rate. Dropout $0.4$ is
applied to MLP hidden activations.
Checkpoints are selected by overall validation loss, using at most
256 validation windows for selection.

\paragraph{GDN-like and RWKV-like BL variants.}
These names describe the components surrounding the Bayesian
recurrence. Both BL variants have four heads with key and value
dimensions 128 and use learned diagonal state decay together with
input-dependent process and observation noise. Initial decay
half-lives span 32 to 4,096 tokens. The initial process and prior
variances are calibrated on training windows to give an observation
gain of approximately $0.5$, with initial observation variance one;
eight disjoint training windows check the calibration. The initial
prior covariance is then fixed, while the process and observation
noise parameters remain trainable.

The GDN-like variant retains GDN's query, key, and value projections,
length-four causal convolutions with SiLU, $\ell_2$ normalization of
queries and keys, and gated RMS-normalized output. It uses RMS
normalization around the residual blocks and a SwiGLU MLP of hidden
width 1,536. The RWKV-like variant instead retains RWKV-7's
two-token mixing, per-head group normalization, local output term,
and output gate, with LayerNorm and a squared-ReLU MLP of hidden
width 2,048. Both the RWKV-like BL and the RWKV-7 baseline disable
cross-layer value sharing. This gives each BL a corresponding
baseline with the same surrounding block structure.

\paragraph{Other model configurations.}
GDN, GDN-2, MesaNet, and Palimpsa-D use four heads of dimension 128
and SwiGLU MLPs of hidden width 1,536. GDN-2 uses independent erase,
write, and decay gates. MesaNet uses 30 conjugate-gradient iterations,
with its regularizer initialized to one and bounded below by $0.25$.
Palimpsa-D uses chunk size 16, a rank-32
parameterization of its update control, and fixed prior precision
one. Mamba-2 uses expansion factor two, 16 heads with value dimension
64 and state dimension 128, one input/output state-projection group,
a length-four causal convolution, and a SwiGLU MLP of hidden width
1,536. Attention uses eight heads of dimension 64, rotary position
embeddings, and the same MLP width. These choices retain the
architectures' surrounding components; recurrent-state sizes and
parameter counts are not constrained to be equal.

\paragraph{Associative-recall evaluation.}
Following the rare-pair analysis of Zoology and Based
\citep{arora2024zoology,arora2024simple}, a target qualifies for
associative recall if it completes a token pair whose training-corpus
count is at most 16 and whose most recent earlier occurrence is
within 1,024 tokens in the same context. The earlier occurrence
must be within the same document; tracking resets at document
boundaries. Distance is measured to that most recent occurrence,
with bins 1--64, 65--256, and 257--1,024 tokens.
All models use the same evaluated token positions. The test
evaluation contains 325,473 target positions, including 4,301 AR
targets; the three distance bins contain 1,669, 1,432, and 1,200
targets, respectively. We compute each metric for each training
seed, then report its arithmetic mean and sample standard deviation.
In particular, reported mean perplexities average the per-seed
perplexities. \Cref{tab:from_scratch_wikitext} gives the full results.

\begin{table}[t]
  \centering
  \scriptsize
  \setlength{\tabcolsep}{3.5pt}
  \renewcommand{\arraystretch}{1.18}
  \caption{\textbf{WikiText-103 test results.} Means $\pm1$ sample standard deviation over three training seeds. LM and AR perplexity are measured on all evaluated tokens and on associative-recall targets, respectively. The final three columns report AR accuracy by distance to the most recent occurrence. All models use BPE tokens, so perplexities are comparable within this table rather than with word-level WikiText-103 results. RWKV-7 and the RWKV-like BL both disable value sharing.}
  \label{tab:from_scratch_wikitext}
  \begin{tabular}{@{}lrrrrrr@{}}
    \toprule
    Model & LM PPL $\downarrow$ & AR PPL $\downarrow$ & \multicolumn{4}{c}{AR accuracy (\%) $\uparrow$} \\
    \cmidrule(l){4-7}
    & & & All & 1--64 & 65--256 & 257--1,024 \\
    \midrule
    BL (GDN-like) & $18.94\pm0.18$ & $8.63\pm0.55$ & $60.09\pm1.36$ & $70.66\pm1.62$ & $64.13\pm1.00$ & $40.58\pm1.96$ \\
    GDN & $18.58\pm0.17$ & $9.29\pm1.20$ & $58.66\pm1.42$ & $71.20\pm1.75$ & $62.92\pm1.33$ & $36.14\pm2.18$ \\
    GDN-2 & $17.85\pm0.12$ & $8.20\pm0.54$ & $60.16\pm0.73$ & $73.84\pm0.93$ & $63.24\pm0.41$ & $37.47\pm2.75$ \\
    MesaNet & $20.69\pm0.23$ & $12.78\pm1.87$ & $52.53\pm2.64$ & $65.23\pm1.77$ & $57.38\pm2.99$ & $29.08\pm4.02$ \\
    Palimpsa & $18.61\pm0.20$ & $10.30\pm1.60$ & $56.34\pm2.41$ & $69.14\pm3.47$ & $61.29\pm2.76$ & $32.61\pm1.38$ \\
    Mamba-2 + MLP & $18.61\pm0.05$ & $20.02\pm1.60$ & $45.82\pm1.19$ & $58.92\pm0.66$ & $50.95\pm1.64$ & $21.47\pm1.71$ \\
    BL (RWKV-like) & $18.12\pm0.14$ & $6.37\pm0.24$ & $63.84\pm0.63$ & $75.77\pm1.12$ & $68.13\pm0.35$ & $42.11\pm0.38$ \\
    RWKV-7 & $17.49\pm0.06$ & $7.31\pm0.77$ & $60.86\pm2.08$ & $73.94\pm1.65$ & $63.97\pm2.09$ & $38.97\pm3.00$ \\
    Attention & $19.35\pm0.08$ & $3.92\pm0.18$ & $74.30\pm0.64$ & $83.20\pm0.57$ & $75.49\pm0.55$ & $60.50\pm1.24$ \\
    \bottomrule
  \end{tabular}
\end{table}

\section{Distillation and Long-Context Retrieval}
\label{app:slimpajama_details}

\paragraph{Model and layerwise distillation.}
All variants start from the released 340M Gated DeltaNet checkpoint,
\texttt{linear-moe-hub/}\allowbreak\texttt{Gated-Deltanet-340M}. It has
24 layers, hidden width 1,024, and four heads of dimension 256. We replace
the sequence mixers at zero-based indices
$\mathcal{I}=\{5,11,17,23\}$ with Bayesian Layers using rotational
dynamics and no short convolution. Each BL head uses 512 real state
coordinates and one shared $512\times512$ covariance matrix.
Each replacement is fitted to the
corresponding frozen GDN mixer's outputs on frozen teacher inputs:
\begin{equation}
  \mathcal{L}^{(\ell)}
  =
  \frac{\operatorname{MSE}\!\left(Y_{\ell}^{\mathrm{BL}},Y_{\ell}^{\mathrm{GDN}}\right)}
       {\operatorname{Var}\!\left(Y_{\ell}^{\mathrm{GDN}}\right)+10^{-6}}
  +\frac{0.05}{B}\sum_{b=1}^{B}
  \left(1-\cos\!\left(y_{\ell,b}^{\mathrm{BL}},y_{\ell,b}^{\mathrm{GDN}}\right)\right),
  \label{eq:spj-phase-a}
\end{equation}
Here, $Y_{\ell}^{\bullet}$ denotes the full batch of mixer outputs at layer
$\ell$, and $y_{\ell,b}^{\bullet}$ denotes sequence $b$'s output flattened
over tokens and features. MSE and teacher-output variance are computed
over all batch, token, and feature entries, separately at each fitted layer.
The cosine term is computed per sequence and averaged over the $B$ sequences
in the batch.
The MLPs and the remaining backbone parameters are frozen. The four
trained replacements are assembled without subsequent joint fine-tuning.

We train each replacement for 1,526 updates on 100,007,936 input tokens,
using 2,048-token SlimPajama sequences and a global batch of 65,536
tokens. AdamW uses learning rate $5\times10^{-4}$,
$\beta=(0.9,0.999)$, $\epsilon=10^{-8}$, no weight decay, gradient
clipping at 1, 25 warmup updates, and cosine decay to $0.1$ times the
peak learning rate. The runs use mixed precision and three training
seeds, varying initialization and document order. We evaluate the
terminal checkpoint of each run.

\paragraph{Continued-GDN control and token accounting.}
The continuation starts from the same original GDN checkpoint and
updates the mixers and MLPs at the four selected layers with ordinary
next-token cross-entropy. All other parameters remain frozen. It uses
763 updates of 131,072 tokens, or 100,007,936 input tokens per run, at
sequence length 2,048. AdamW uses learning rate $3\times10^{-4}$,
$\beta=(0.9,0.95)$, $\epsilon=10^{-8}$, no weight decay, gradient
clipping at 1, 50 warmup updates, and cosine decay to $0.1$ times the
peak learning rate. We use three training seeds and evaluate terminal
checkpoints.

Each selected mixer therefore sees approximately 100M training tokens
under either procedure, giving 400M token exposures summed across the
four mixers. BL uses a separately shuffled stream at each replacement
site. Continued GDN processes one stream through all four selected
layers, reusing the corresponding BL seed's layer-5 stream. This
matches token exposure per selected mixer; the objectives, trainable
parameters, and total computation differ.

\paragraph{Retrieval evaluation.}
We evaluate six RULER needle-in-a-haystack tasks at context lengths
2,048, 4,096, 8,192, and 16,384. Each task and length has 200 examples,
giving 4,800 examples per model and training seed. All methods use the
same prompts and evaluation batches. Decoding uses cached full-context
prefill followed by greedy generation, stopping at EOS or 128 generated
tokens. We retain RULER's fractional credit for multi-answer tasks.
\Cref{tab:spj-ruler-full} reports the full task breakdown. The tables
report sample standard deviations over training seeds; the
frozen GDN reference is a single checkpoint evaluated on the same
examples.

For the heatmaps, we divide the recorded answer position by the prompt
length and group examples into ten equal-width position bins. Within
each bin, the aggregate panel averages the examples across all six
tasks, and each task panel uses only that task's examples. Colors show
the difference between the BL seed mean and the indicated reference.
\Cref{fig:slimpajama} compares BL with continued GDN;
\Cref{fig:spj-frozen-position} uses the original frozen checkpoint.

\begin{figure}[!ht]
  \centering
  \includegraphics[width=\linewidth]{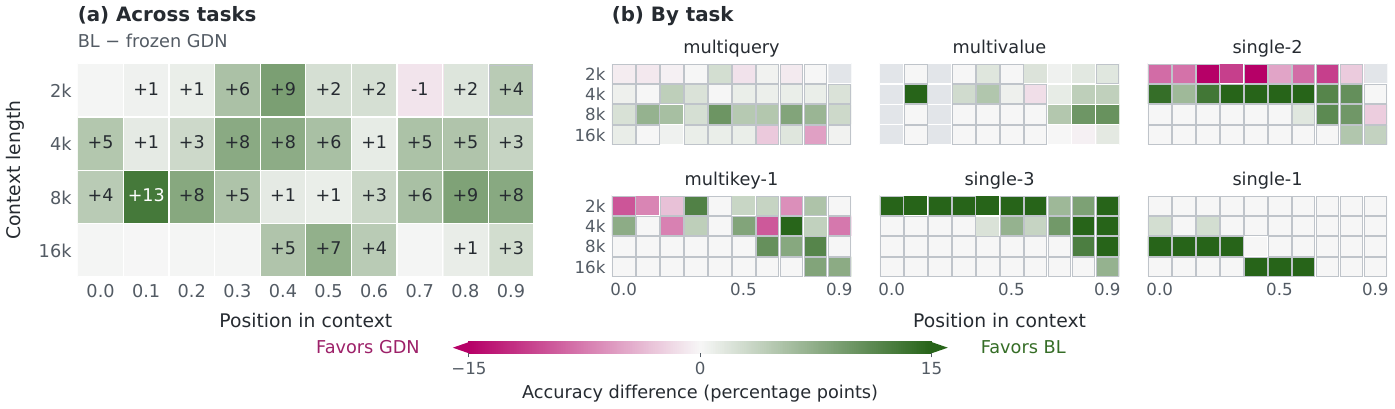}
  \caption{\textbf{Long-context retrieval relative to frozen GDN.}
    BL minus frozen-GDN retrieval accuracy on RULER, \textbf{(a)}
    aggregated across six needle-in-a-haystack tasks and \textbf{(b)}
    separately for each task. BL averages three training seeds; frozen
    GDN is the original checkpoint. The evaluation examples, layout,
    and color scale match \Cref{fig:slimpajama}. Gray outlines mark
    cells containing fewer than 30 examples; gray fill indicates empty cells.}
  \label{fig:spj-frozen-position}
\end{figure}

\begin{table}[!ht]
\caption{\textbf{RULER retrieval accuracy (\%).} BL and continued GDN report mean $\pm$ sample standard deviation over three training seeds; frozen GDN is a single checkpoint. Each task and context length uses the same 200 examples for all methods. The final column averages the six tasks.}
\label{tab:spj-ruler-full}
\begin{center}
\scriptsize
\setlength{\tabcolsep}{3pt}
\renewcommand{\arraystretch}{1.15}
\resizebox{\linewidth}{!}{%
\begin{tabular}{llrrrrrrr}
\toprule
Context & Model & Single-1 & Single-2 & Single-3 & Multikey-1 & Multivalue & Multiquery & Mean \\
\midrule
2048 & Frozen GDN & $100.00$ & $99.00$ & $26.50$ & $27.50$ & $22.88$ & $24.62$ & $50.08$ \\
 & Continued GDN & $100.00 \pm 0.00$ & $97.50 \pm 1.80$ & $26.33 \pm 3.62$ & $27.00 \pm 1.32$ & $23.88 \pm 0.54$ & $24.75 \pm 0.33$ & $49.91 \pm 0.73$ \\
 & BL & $100.00 \pm 0.00$ & $89.83 \pm 2.52$ & $49.33 \pm 9.50$ & $27.00 \pm 1.50$ & $24.08 \pm 0.14$ & $24.33 \pm 0.19$ & $52.43 \pm 1.50$ \\
\midrule
4096 & Frozen GDN & $99.00$ & $50.50$ & $0.50$ & $22.50$ & $16.75$ & $22.75$ & $35.33$ \\
 & Continued GDN & $99.17 \pm 0.29$ & $49.00 \pm 2.65$ & $1.17 \pm 1.15$ & $23.50 \pm 0.50$ & $18.29 \pm 0.81$ & $22.79 \pm 0.19$ & $35.65 \pm 0.57$ \\
 & BL & $99.50 \pm 0.50$ & $66.50 \pm 7.94$ & $5.83 \pm 2.75$ & $24.50 \pm 0.50$ & $19.54 \pm 3.72$ & $24.04 \pm 0.19$ & $39.99 \pm 2.19$ \\
\midrule
8192 & Frozen GDN & $76.00$ & $17.50$ & $1.00$ & $11.50$ & $4.75$ & $8.75$ & $19.92$ \\
 & Continued GDN & $75.83 \pm 5.86$ & $18.17 \pm 1.15$ & $1.00 \pm 0.50$ & $12.67 \pm 1.26$ & $6.17 \pm 1.13$ & $9.21 \pm 1.00$ & $20.51 \pm 1.27$ \\
 & BL & $88.33 \pm 2.02$ & $19.67 \pm 0.58$ & $4.17 \pm 0.76$ & $15.33 \pm 0.58$ & $11.38 \pm 1.95$ & $14.17 \pm 0.44$ & $25.51 \pm 0.41$ \\
\midrule
16384 & Frozen GDN & $33.50$ & $7.50$ & $2.00$ & $5.00$ & $3.50$ & $5.62$ & $9.52$ \\
 & Continued GDN & $35.17 \pm 3.88$ & $7.00 \pm 1.32$ & $2.50 \pm 0.50$ & $5.00 \pm 0.00$ & $3.33 \pm 0.19$ & $5.33 \pm 0.69$ & $9.72 \pm 0.87$ \\
 & BL & $44.17 \pm 2.93$ & $8.33 \pm 1.04$ & $2.67 \pm 0.29$ & $6.17 \pm 0.29$ & $3.88 \pm 0.12$ & $5.58 \pm 0.26$ & $11.80 \pm 0.58$ \\
\bottomrule
\end{tabular}%
}
\end{center}
\end{table}

\paragraph{Language-model evaluation.}
\Cref{tab:spj-ppl-full} reports perplexity on WikiText-103 test,
held-out SlimPajama, and PG-19 validation. We use the released GDN
tokenizer and identical token arrays and window boundaries for all
models. WikiText-103 concatenates the nonempty test rows without
additional BOS or EOS tokens. SlimPajama uses a held-out pool and the
first complete windows reaching at least one million predicted tokens
at each context length. PG-19 uses the first 50 validation books with
EOS separators, capped at three million tokens. Each corpus is split
into nonoverlapping windows, with $L-1$ next-token targets per window
of length $L$. We compute corpus cross-entropy over all predicted
tokens and exponentiate it to obtain each model's perplexity, then
report the arithmetic mean and sample standard deviation over training
seeds.

\begin{table}[!ht]
\caption{\textbf{Language-model perplexity after adaptation.} BL and continued GDN report mean $\pm$ sample standard deviation over three training seeds; frozen GDN is a single checkpoint. All methods use identical token windows within each corpus and context length. Targets gives the number of predicted tokens per model and seed.}
\label{tab:spj-ppl-full}
\begin{center}
\small
\setlength{\tabcolsep}{4pt}
\renewcommand{\arraystretch}{1.15}
\begin{tabular}{llrrrr}
\toprule
Corpus & Context & Frozen GDN & Continued GDN & BL & Targets \\
\midrule
WikiText-103 & 2048 & $15.942$ & $16.076 \pm 0.024$ & $16.684 \pm 0.061$ & 315,238 \\
 & 4096 & $15.063$ & $15.190 \pm 0.027$ & $15.747 \pm 0.054$ & 315,315 \\
 & 8192 & $14.566$ & $14.697 \pm 0.032$ & $15.208 \pm 0.052$ & 311,258 \\
 & 16384 & $14.438$ & $14.583 \pm 0.036$ & $15.054 \pm 0.047$ & 311,277 \\
\midrule
SlimPajama & 2048 & $10.938$ & $10.994 \pm 0.005$ & $11.311 \pm 0.010$ & 1,000,983 \\
 & 4096 & $10.633$ & $10.686 \pm 0.005$ & $10.987 \pm 0.010$ & 1,003,275 \\
 & 8192 & $10.512$ & $10.564 \pm 0.005$ & $10.857 \pm 0.009$ & 1,007,493 \\
 & 16384 & $10.476$ & $10.527 \pm 0.005$ & $10.817 \pm 0.008$ & 1,015,746 \\
\midrule
PG-19 & 2048 & $20.295$ & $20.651 \pm 0.288$ & $21.343 \pm 0.080$ & 2,996,808 \\
 & 4096 & $19.637$ & $19.973 \pm 0.271$ & $20.643 \pm 0.070$ & 2,997,540 \\
 & 8192 & $19.290$ & $19.617 \pm 0.265$ & $20.265 \pm 0.066$ & 2,997,906 \\
 & 16384 & $19.139$ & $19.463 \pm 0.262$ & $20.093 \pm 0.063$ & 2,998,089 \\
\bottomrule
\end{tabular}
\end{center}
\end{table}

\paragraph{Existing assets and licenses.}
The experiments use public research assets only. Zoology/MQAR
\citep{arora2024zoology} and RULER \citep{hsieh2024ruler} are released
under Apache-2.0 licenses. The
\texttt{linear-moe-hub/Gated-Deltanet-340M} checkpoint
\citep{du2026mom,yang2025gated} is released on Hugging Face under
Apache-2.0. SlimPajama
\citep{cerebras2023slimpajama} is an open dataset release whose
documentation lists source-specific terms for its constituent subsets.
WikiText-103 \citep{merity2017pointer} is distributed under
CC-BY-SA-3.0/GFDL terms, and PG-19 \citep{rae2020compressive} is
released from an Apache-2.0 repository with source texts drawn from
Project Gutenberg. We cite the original creators and do not redistribute
the underlying assets in this paper.

\end{document}